%% file: bayesian.tex
\newif\ifisarxiv
\isarxivtrue

\ifisarxiv
\documentclass[11pt]{article}
\usepackage{fullpage}
\else
\documentclass{article}
\PassOptionsToPackage{numbers,compress,sort}{natbib}
\usepackage[nonatbib]{sty/neurips2019/neurips_2019}
\fi

\usepackage[utf8]{inputenc} 
\usepackage[T1]{fontenc}    
\usepackage{hyperref}       
\usepackage{url}            
\usepackage{booktabs}       
\usepackage{amsfonts}       
\usepackage{amsmath}
\usepackage{nicefrac}       
\usepackage{xfrac}
\usepackage{microtype}      
\usepackage{color}
\usepackage{graphicx}
\usepackage{tikz}
\usepackage{wrapfig}
\usepackage{algorithm}
\usepackage{algorithmic}
\usepackage{caption}

\input{shortdefs}

\usepackage{grffile}

\title{Bayesian experimental design using regularized \\
  determinantal point processes}

\ifisarxiv\date{}\def\And{\and}\fi
\author{%
        Micha{\l } Derezi\'{n}ski \\
Department of Statistics\\
University of California, Berkeley\\
\texttt{mderezin@berkeley.edu}\\
\And
Feynman Liang \\
Department of Statistics\\
University of California, Berkeley\\
\texttt{feynman.liang@gmail.com}
\And
 Michael W. Mahoney\\
ICSI and Department of Statistics\\
University of California, Berkeley\\
\texttt{mmahoney@stat.berkeley.edu}
}

\begin{document}

\maketitle

\begin{abstract}
In experimental design, we are given $n$ vectors in $d$ dimensions, and
our goal is to select $k\ll n$ of them to perform expensive
measurements, e.g., to obtain labels/responses, for a linear regression task. Many statistical criteria have
been proposed for choosing the optimal design, with popular choices including
A- and D-optimality. If prior knowledge is given, typically in the form of a
$d\times d$ precision matrix $\A$, then all of the criteria can
be extended to incorporate that information via a Bayesian
framework. In this paper, we demonstrate a new fundamental connection
between Bayesian experimental design and determinantal point
processes, the latter being widely used for sampling diverse subsets of
data. We use this connection to develop new efficient algorithms for
finding $(1+\epsilon)$-approximations of optimal designs under four
optimality criteria: A, C, D and V. Our algorithms can achieve this
when the desired subset size $k$ is $\Omega(\frac{d_{\A}}{\epsilon} +
\frac{\log 1/\epsilon}{\epsilon^2})$, where $d_{\A}\leq d$ is the
$\A$-effective dimension, which can often be much smaller than $d$. Our
results offer direct improvements over a number of prior works, for both
Bayesian and classical experimental design, in terms of
algorithm efficiency, approximation quality, and range
of applicable criteria.
\end{abstract}

\section{Introduction}

Consider a collection of $n$ experiments parameterized by
$d$-dimensional vectors $\x_1,\dots,\x_n$, and let $\X$ denote the
$n\times d$ matrix with rows $\x_i^\top$. The outcome of the $i$th
experiment is a random variable $y_i = \x_i^\top\w + \xi_i$, where
$\w$ is the parameter vector of a linear model with prior distribution
$\Nc(\zero,\sigma^2\A^{-1})$, and $\xi_i\sim \Nc(0,\sigma^2)$ is
independent noise. In experimental design, we have access to the vectors $\x_i^\top$, for $i\in\{1,\ldots,n\}$, but we are allowed to
observe only a small number of outcomes $y_i$ for experiments we choose.
Suppose that we observe the outcomes from a subset
$S\subseteq\{1,...,n\}$ of $k$ experiments. The
posterior distribution of $\w$ given $\y_S$ (the vector of outcomes in $S$) is:
\begin{align*}
  \w\mid \y_S \ \sim\ \Nc\Big(\ (\X_S^\top\X_S + \A)^{-1}\X_S^\top\y_S
  ,\ \ \sigma^2(\X_S^\top\X_S+\A)^{-1}\ \Big),
\end{align*}
where $\X_S$ denotes the $k\times d$ matrix with rows $\x_i^\top$ for
$i\in S$.
In the Bayesian framework of experimental design \cite{bayesian-design-review}, we assume
that the prior precision matrix $\A$ of the linear model $\w$ is known,
and our goal is to choose $S$ so as to minimize some quantity (a.k.a.~an
optimality criterion) measuring the ``size'' of the
posterior covariance matrix $\Sigmab_{\w\mid\y_S} =
\sigma^2(\X_S^\top\X_S+\A)^{-1}$.
This quantity is a function of the subset
covariance $\X_S^\top\X_S$. Note that if matrix $\A$ is
non-invertible then, even though the prior distribution is
ill-defined, we can still interpret it as having no prior
information in the directions with eigenvalue 0. In particular, for
$\A=\zero$ we recover classical experimental design, where the
covariance matrix of $\w$ given $\y_S$ is
$\sigma^2(\X_S^\top\X_S)^{-1}$.
We will write the Bayesian optimality
criteria as functions $f_{\A}(\Sigmab)$, where $\Sigmab$ corresponds
to the subset covariance $\X_S^\top\X_S$. The following standard
criteria \cite{bayesian-design-review,optimal-design-pukelsheim} are of primary interest to us:
\begin{enumerate}
\item A-optimality:\quad $f_{\A}(\Sigmab) = \tr\big((\Sigmab+\A)^{-1}\big)$;
\item C-optimality:\quad $f_{\A}(\Sigmab) = \cb^\top(\Sigmab+\A)^{-1}\cb$ for
  some vector $\cb$;
\item D-optimality:\quad $f_{\A}(\Sigmab) = \det(\Sigmab+\A)^{-1/d}$;
  \item V-optimality:\quad $f_{\A}(\Sigmab) =
    \frac1n\tr\big(\X(\Sigmab+\A)^{-1}\X^\top\big)$.
  \end{enumerate}
Other popular criteria (less relevant to our discussion) include E-optimality,
$f_{\A}(\Sigmab)=\|(\Sigmab+\A)^{-1}\|$ (here, $\|\cdot\|$ denotes the
spectral norm) and G-optimality,
$f_{\A}(\Sigmab)=\max\diag(\X(\Sigmab+\A)^{-1}\X^\top)$.

The general task we consider is given as follows, where $[n]$
denotes $\{1,...,n\}$:

\textbf{Bayesian experimental design.}
Given an $n\times d$ matrix $\X$,
a criterion $f_{\A}(\cdot)$ and~$k\in[n]$,
\begin{align*}
  \text{efficiently minimize }\quad
  f_{\A}(\X_S^\top\X_S) \quad\text{over }\ S\subseteq [n]
  \ \text{ s.t. }\ |S|= k.
\end{align*}

\textbf{Optimal value.}
Given $\X$, $f_{\A}$ and $k$, we denote the optimum as $\opt =
\min_{S:|S|=k}f_{\A}(\X_S^\top\X_S)$.

The prior work around this problem can be grouped into two research questions.
The first question asks what can
we infer about $\opt$ just from the spectral
information about the problem, which is contained in the data covariance matrix
$\Sigmab_\X=\X^\top\X\in\R^{d\times d}$. The second question asks when does there exist a
polynomial time algorithm for finding a $(1+\epsilon$)-approximation
for $\opt$.

\textbf{Question 1:} \quad
Given only $\Sigmab_{\X}$, $f_{\A}$ and $k$,
what is the upper bound on $\opt$?

\textbf{Question 2:} \quad
Given $\X$, $f_{\A}$ and $k$, can we efficiently find a $(1+\epsilon)$-approximation for
$\opt$?

A key aspect of both of these questions is how large the subset
size $k$ has to be for us to provide useful answers. As a baseline, we
should expect meaningful results when $k$ is at least $\Omega(d)$ (see
discussion in \cite{near-optimal-design}), and in fact,
for classical experimental design (i.e., when $\A=\zero$), the problem becomes ill-defined
when $k<d$. In the Bayesian setting we should be able to exploit the
additional prior knowledge 
to achieve strong results even for $k\ll d$. Intuitively, the larger
the prior precision matrix $\A$, the fewer degrees of freedom we have
in the problem. To measure this, we use the statistical notion of
\emph{effective dimension} \cite{ridge-leverage-scores}.
\begin{definition}
For $d\times d$ psd matrices $\A$ and $\Sigmab$,
  let the $\A$-effective dimension of $\Sigmab$
be defined as $d_{\A}(\Sigmab) =
\tr\big(\Sigmab(\Sigmab+\A)^{-1}\big) \leq d$.
We will use the shorthand $d_{\A}$ when referring to $d_{\A}(\Sigmab_{\X})$.
\end{definition}

Recently, \cite{regularized-volume-sampling} obtained bounds on Bayesian
A/V-optimality criteria for $k\geq d_{\A}$, suggesting that $d_{\A}$ is the right notion of
degrees of freedom for this problem. We argue that $d_{\A}$ can in fact
be far too large of an estimate because it does not take into account
the size $k$ when computing the effective dimension. Intuitively, since
$d_{\A}$ is computed using the full data covariance $\Sigmab_{\X}$, it
is not in the appropriate scale with respect to the smaller covariance
$\X_S^\top\X_S$.
One way to correct this is to increase the regularization on
$\Sigmab_\X$ from $\A$ to $\frac{n}{k}\A$ and use
$d_{\frac nk\A}=d_{\frac{n}{k}\A}(\Sigmab_\X)$ as the degrees of
freedom. Another way is to rescale the full covariance to 
$\frac{k}{n}\Sigmab_\X$ and use $d_\A( \frac{k}{n} \Sigmab_\X)$ as the
degrees of freedom. In fact, since $d_{\frac{n}{k}\A}(\Sigmab_\X) = d_\A(
\frac{k}{n} \Sigmab_\X)$, these two approaches are identical.  
Note that $d_{\frac
  nk\!\A}\leq d_{\A}$ and this gap can be very large for some
problems (see discussion in Appendix \ref{a:deff}). 
The following result supports the above reasoning by showing that for any
$k$ such that $k\geq 4d_{\frac nk\!\A}$, there is $S$ of size $k$
which satisfies $f_{\A}(\X_S^\top\X_S) = O(1)\cdot f_{\A}(\frac kn\Sigmab_{\X})$.
This not only improves on \cite{regularized-volume-sampling} in terms
of the supported range of sizes $k$, but also in terms of the obtained bound (see
Section \ref{s:related-work} for a~comparison).
\begin{theorem}\label{t:q1}
  Let $f_{\A}$ be A/C/D/V-optimality and $\X$ be
  $n\times d$. For any $k$ such that $k\geq 4d_{\frac
    nk\!\A}$, 
  \begin{align*}
\opt    \leq \bigg(\ 1+8\,\frac{d_{\frac
    nk\!\A}}{k} + 8\sqrtshort{\frac{\ln (k/d_{\frac nk\!\A})}{k}}\
    \bigg)\cdot f_{\A}\big(\tfrac kn\Sigmab_{\X}\big).
  \end{align*}
\end{theorem}
\begin{remark}
We give an $O(ndk+k^2d^2)$ time algorithm for finding subset $S$ that certifies
  this bound.
\end{remark}
To establish Theorem~\ref{t:q1}, we propose
a new sampling distribution $\DPPreg{p}(\X,\A)$, where
$p=(p_1,...,p_n)\in[0,1]^n$ is a vector of weights. This is a special
\emph{regularized} variant of a determinantal point
process (DPP), which is a well-studied family of distributions
\cite{dpp-ml} with numerous applications in sampling diverse subsets
of elements. Given a psd matrix $\A$ and a weight vector $p$, we
define $\DPPreg{p}(\X,\A)$ as a distribution over subsets $S\subseteq[n]$ (of all
sizes) such that:
\begin{align*}
(\text{see Def.~\ref{d:r-dpp}})\qquad  \Pr(S)\propto \det(
  \X_S^\top\X_S+\A)\ \cdot \prod_{i\in S}p_{i}\cdot\prod_{i\not\in
  S}(1-p_i).
\end{align*}
A number of regularized DPPs have been proposed recently
\cite{dpp-intermediate,regularized-volume-sampling}, mostly within the context of Randomized Numerical Linear Algebra (RandNLA)~\cite{Mah-mat-rev_JRNL,DM16_CACM,RandNLA_PCMIchapter_TR}.
To our knowledge, ours is the first such definition that strictly falls
under the traditional definition of a DPP \cite{dpp-ml}.
We show this in Section \ref{s:r-dpp}, where we also prove that
regularized DPPs can be decomposed into a low-rank DPP plus
i.i.d. Bernoulli sampling (Theorem \ref{t:algorithm}). This decomposition reduces
the sampling cost from $O(n^3)$ to $O(nd^2)$,
and involves a more general result about DPPs defined via a
correlation kernel (Lemma \ref{l:decomposition}), which is of
independent interest.

To prove Theorem \ref{t:q1}, in Section \ref{s:guarantees} we demonstrate a fundamental connection between an
$\A$-regularized DPP and Bayesian experimental design with precision
matrix $\A$. For simplicity of exposition, let the weight vector $p$ be uniformly equal $(\frac kn,...,\frac kn)$. If
$S\sim\DPPreg{p}(\X,\A)$ and $f_{\A}$ is any one of the
A/C/D/V-optimality criteria, then:
\begin{align}
\text{(a)}\ \ \E\big[f_{\A}(\X_S^\top\X_S)\big] \leq f_{\A}\big(\tfrac
  kn\Sigmab_{\X}\big)\quad\text{and}\quad\text{(b)}\ \ \E\big[|S|\big]
  \leq d_{\frac nk\!\A}+k. \label{eq:f-ineq}
\end{align}
Theorem \ref{t:q1} follows
by showing an inequality similar to (\ref{eq:f-ineq}a) when $\DPPreg{p}(\X,\A)$ is
restricted to subsets of size at most $k$ (proof in
Section~\ref{s:guarantees}).
When $\A=\zero$, then $\DPPreg{p}(\X,\A)$ bears a lot of
similarity to \emph{proportional volume sampling} which is an
(unregularized) determinantal distribution proposed by
\cite{proportional-volume-sampling}. That work used an inequality similar to
(\ref{eq:f-ineq}a) for obtaining
$(1+\epsilon)$-approximate algorithms in A/D-optimal classical experimental
design (Question 2 with $\A=\zero$). However,
the algorithm of \cite{proportional-volume-sampling} for proportional
volume sampling takes $O(n^4dk^2\log k)$ time, making it practically
infeasible. On the other hand, the time complexity of sampling from
$\DPPreg{p}(\X,\A)$ is only $O(nd^2)$, and recent advances in RandNLA for DPP
sampling
\cite{leveraged-volume-sampling,correcting-bias,dpp-intermediate}
suggest that $O(nd\log n + \poly(d))$ time
is also possible. Extending the ideas of
\cite{proportional-volume-sampling} to our new regularized DPP
distribution, we obtain efficient $(1+\epsilon)$-approximation
algorithms for A/C/D/V-optimal Bayesian experimental design.

\begin{theorem}\label{t:q2}
  Let $f_{\A}$ be A/C/D/V-optimality and $\X$ be $n\times d$. If
  $k=\Omega\big(\frac{d_{\A}}{\epsilon} +
  \frac{\log1/\epsilon}{\epsilon^2}\big)$ for some $\epsilon\in(0,1)$, then there is a polynomial time
  algorithm that finds $S\subseteq [n]$ of size $k$ such that
  \begin{align*}
    f_{\A}\big(\X_S^\top\X_S\big) \leq (1+\epsilon)\cdot\opt.
  \end{align*}
\end{theorem}
\begin{remark}
  The algorithm referred to in Theorem~\ref{t:q2} first solves a convex relaxation of the task via a
  semi-definite program (SDP) to find the weights $p\in[0,1]^n$,
  then samples from the $\DPPreg{p}(\X,\A)$ distribution
  $O(1/\epsilon)$ times. The expected cost in addition to the SDP is $O(ndk+k^2d^2)$.
\end{remark}

\begin{wrapfigure}{r}{0.6\textwidth}
  \vspace{-4mm}
  \centering
  \renewcommand{\arraystretch}{1.5}
\begin{tabular}{r||c|c|l}
 &Criteria&Bayesian&$k=\Omega(\cdot)$\\
  \hline\hline
\small  \cite{tractable-experimental-design}
  &\small A,V&\xmark&$\frac{d^2}{\epsilon}$\\
\small  \cite{near-optimal-design}
  &\small A,C,D,E,G,V&\cmark&$\frac {d}{\epsilon^2}$\\
\small  \cite{proportional-volume-sampling}
  &\small A,D&\xmark&$\frac{d}{\epsilon} +
  \frac{\log1/\epsilon}{\epsilon^2}$\\
  \hline
\small\textbf{this paper}&\small A,C,D,V&\cmark& $\frac{d_{\A}}{\epsilon} +
  \frac{\log1/\epsilon}{\epsilon^2}$
\end{tabular}
\vspace{2mm}
\captionof{table}{Comparison of SDP-based $(1+\epsilon)$-approximation algorithms for
  classical and Bayesian experimental design (X-mark means that only
  the classical setting applies).}
\label{tab:comp}
\vspace{-3mm}
\end{wrapfigure}
Note that unlike in Theorem \ref{t:q2} we use the unrescaled effective
dimension $d_{\A}$ instead of the rescaled one, $d_{\frac nk\!\A}$. The
actual effective dimension that applies here (given in the proof in
Section \ref{s:guarantees}) depends on the SDP solution. It is
always upper bounded by $d_{\A}$, but it may be significantly
smaller. This result is a direct extension of
\cite{proportional-volume-sampling} to Bayesian setting and to
C/V-optimality criteria. Moreover, in their case, proportional
volume sampling is usually the computational bottleneck (because
its time dependence on the dimension can reach $O(d^{11})$), whereas for us
the cost of sampling is negligible compared to the SDP. A
number of different methods can be used to solve the SDP relaxation (see
Section \ref{s:experiments}). For
example, \cite{near-optimal-design} suggest using an iterative
optimizer called entropic mirror
descent, which is known to exhibit fast convergence and
can run in $O(nd^2T)$ time, where $T$ is the number of iterations.

\section{Related work}
\label{s:related-work}

We first discuss the prior works that focus on bounding the experimental
design optimality criteria without obtaining
$(1+\epsilon)$-approximation algorithms. First non-trivial
bounds for the \emph{classical} A-optimality criterion (with $\A=\zero$) were shown by
\cite{avron-boutsidis13}. Their result implies that for any $k\geq d$,
$\opt\leq
(1+ \frac{d-1}{k-d+1})\cdot f_{\zero}(\frac kn\Sigmab_\X)$ and they
provide polynomial time algorithms for finding such solutions. The result
was later extended by
\cite{unbiased-estimates, regularized-volume-sampling,unbiased-estimates-journal} to the
case where $\A=\lambda\I$, obtaining that for any $k\geq
d_{\lambda\I}$, we have
$\opt\leq (1 +
\frac{d_{\lambda\I}-1}{k-d_{\lambda\I}+1})\cdot f_{\frac
  kn\!\lambda\I}(\frac kn\Sigmab_\X)$, and also a faster $O(nd^2)$ time
algorithm was provided. Their result can be
easily extended to cover any psd matrix $\A$ and V/C-optimality (but not D-optimality). The key
improvements of our Theorem \ref{t:q1} are that we cover a potentially
much wider range of subset sizes, because $d_{\frac nk\!\lambda\I}\leq
d_{\lambda\I}$, and our bound can be much tighter because
$f_{\lambda\I}(\frac kn\Sigmab_\X)\leq f_{\frac kn\!\lambda\I}(\frac
kn\Sigmab_\X)$. Finally, \cite{minimax-experimental-design} propose a
new notion of \emph{minimax} experimental design, which is related to A/V-optimality. They also use a determinantal distribution
for subset selection, however, due to different assumptions, their
bounds are incomparable.

A number of works proposed $(1+\epsilon)$-approximation algorithms
for experimental design which start with solving a convex
relaxation of the problem, and then use some rounding strategy to
obtain a discrete solution (see Table \ref{tab:comp} for comparison). For example,
\cite{tractable-experimental-design} gave an approximation algorithm
for classical A/V-optimality with $k=\Omega(\frac{d^2}{\epsilon})$, where the
rounding is done in a greedy fashion, and some randomized rounding
strategies are also discussed. \cite{proportional-volume-sampling}
suggested \emph{proportional volume sampling} for the rounding step
and obtained approximation algorithms for classical A/D-optimality with
$k=\Omega(\frac d\epsilon+\frac {\log1/\epsilon}{\epsilon^2})$. Their
approach is particularly similar to ours (when $\A=\zero$). However,
as discussed earlier, while their algorithms are polynomial, they are
virtually intractable. \cite{near-optimal-design} proposed
an efficient algorithm with a $(1+\epsilon)$-approximation guarantee for a wide
range of optimality criteria, including A/C/D/E/V/G-optimality, both classical and Bayesian, when
$k=\Omega(\frac d{\epsilon^2})$.  Our results improve on this work in
two ways: (1) in terms of the dependence on $\epsilon$ for
A/C/D/V-optimality, and (2) in
terms of the dependence on the dimension (by replacing $d$ with
$d_{\A}$) in the Bayesian setting. A lower bound shown by
\cite{proportional-volume-sampling} implies that our Theorem \ref{t:q2} cannot be
directly extended to E-optimality, but a similar lower bound does not
exist for G-optimality. We remark that the approximation approaches
relying on a convex relaxation can generally be converted to
an upper bound on $\opt$ akin to our Theorem \ref{t:q1},
however none of them apply to the regime of $k\leq d$,
which is of primary interest in the Bayesian setting.

Purely greedy approximation algorithms have been shown to provide
guarantees in a number of special cases for experimental design. One
example is classical D-optimality criterion, which can be converted to
a submodular function \cite{submodularity-optimal-design}. Also, greedy algorithms
for Bayesian A/V-optimality criteria have been considered
\cite{greedy-supermodular,greedy-graph-sampling}. These
methods can only provide a constant factor approximation guarantee
(as opposed to $1+\epsilon$), and the factor is generally problem
dependent (which means it could be arbitrarily large).
Finally, a number of
heuristics with good empirical performance have been proposed, such as Fedorov's
exchange method \cite{cook1980comparison}. However, in this work
we focus on methods that provide theoretical approximation guarantees.

\section{A new regularized determinantal point process}
\label{s:r-dpp}
  In this section we introduce the determinantal sampling distribution we use for
  obtaining guarantees in Bayesian experimental design.
  Determinantal point processes (DPP) form a family of distributions
  which are used to model repulsion between elements in a random set,
  with many applications in machine learning \cite{dpp-ml}. Here, we
  focus on the setting where we are sampling out of all $2^n$ subsets
  $S\subseteq[n]$. Traditionally, a DPP is defined by a
  correlation kernel, which is an $n\times n$ psd matrix $\K$ with
  eigenvalues between 0 and 1, i.e., such that $\zero\preceq\K\preceq
  \I$. Given a correlation kernel $\K$, the corresponding DPP is
  defined as
  \begin{align*}
    S\sim\DPPcor(\K)\qquad\text{iff}\qquad
    \Pr(T\subseteq S) = \det(\K_{T,T})\ \  \forall_{T\in[n]},
  \end{align*}
where $\K_{T,T}$ is the submatrix of $\K$ with rows and columns
indexed by $T$. Another way of defining a DPP, popular in the machine
learning community, is via an ensemble kernel $\L$. Any psd matrix
$\L$ is an ensemble kernel of a DPP defined as:
  \begin{align*}
    S\sim\DPPens(\L)\qquad\text{iff}\qquad \Pr(S) \propto \det(\L_{S,S}).\quad~
  \end{align*}
Crucially, every $\DPPens$ is also a $\DPPcor$, but not the other way
around. Specifically, we have:
\begin{align*}
 \text{(a)}\ \  \DPPcor(\K) = \DPPens\big(\I -
  (\I+\L)^{-1}\big),\quad\text{and}\quad
  \text{(b)}\ \ \DPPens(\L) = \DPPcor\big(\K(\I-\K)^{-1}\big),
\end{align*}
but (b) requires that $\I-\K$ be invertible which is not true
for some DPPs.
(This will be important in our analysis.)
The classical algorithm for sampling from a DPP requires the
eigendecomposition of either matrix $\K$ or $\L$, which in general
costs $O(n^3)$, followed by a sampling procedure which costs
$O(n\,|S|^2)$ \cite{dpp-independence,dpp-ml}.

We now define our regularized DPP and describe its
connection with correlation and ensemble~DPPs.
\begin{definition}\label{d:r-dpp}
Given matrix $\X\in\R^{n\times d}$, a sequence  $p=(p_1,\dots,p_n)\in[0,1]^n$
and a psd matrix $\A\in\R^{d\times d}$ such that
$\sum_ip_i\x_i\x_i^\top\!+\A$ is
full rank, let
$\DPPreg{p}(\X,\A)$ be a distribution over $S\subseteq [n]$:
  \begin{align}
  \Pr(S) = \frac{\det(
  \X_S^\top\X_S+\A)}{\det\!\big(
 \sum_ip_i\x_i\x_i^\top+ \A\big)}\
\cdot \prod_{i\in S}p_{i}\cdot\prod_{i\not\in S}(1-p_i).\label{eq:poisson-prob}
\end{align}
\end{definition}
The fact that this is a proper distribution (i.e., that it sums to
one) can be restated as a determinantal expectation formula: if
$b_i\sim\mathrm{Bernoulli}(p_i)$ are independent Bernoulli random
variables, then
\begin{align*}
\sum_{S\subseteq [n]}\!\det(
  \X_S^\top\X_S+\A)\prod_{i\in S}p_{i}\prod_{i\not\in S}(1-p_i) =
  \E\bigg[\!\det\!\Big(\sum_ib_i\x_i\x_i^\top\!+\A\Big)\bigg]\overset{(*)}{=}
  \det\!\Big(\sum_i\E[b_i]\x_i\x_i^\top\!+\A\Big),
\end{align*}
where $(*)$ was shown by \cite[Lemma
7]{determinantal-averaging} and $\E[b_i]=p_i$.
Our main result in this section is the following efficient algorithm
for $\DPPreg{p}(\X,\A)$ which reduces it to sampling from a
correlation DPP.
\begin{theorem}\label{t:algorithm}
For any $\X\in\R^{n\times d}$, $p\in[0,1]^n$ and a psd matrix $\A$
s.t.~$\sum_ip_i\x_i\x_i^\top\!+\A$ is
full rank, let
\[T\sim\DPPcor\big(\D_p^{\sfrac12}\X(\A+\X^\top\D_p\X)^{-1}\X^\top\D_p^{\sfrac12}\big),
  \quad\text{where}\quad\D_p=\diag(p).\]
If $b_i\sim\mathrm{Bernoulli}(p_i)$ are independent random variables, then
$T\cup \{i:b_i\!=\!1\}\sim \DPPreg{p}(\X,\A)$.
\end{theorem}
\ifisarxiv
\begin{minipage}{0.6\textwidth}
\else
  \begin{minipage}{0.55\textwidth}
\fi
\begin{remark}
Since the correlation kernel matrix has rank at most $d$, the preprocessing
cost of eigendecomposition is $O(nd^2)$. Then, each sample costs only $O(n\,|T|^2)$.
\end{remark}
\ifisarxiv\vspace{-3mm}\fi
We prove the theorem in three steps. First, we express
$\DPPreg{p}(\X,\A)$ as an ensemble DPP, which requires some additional
assumptions on $\A$ and $p$ to be possible. Then, we convert the
ensemble to a correlation kernel (eliminating the extra assumptions),
and finally show that this kernel can be decomposed into a rank $d$
kernel plus Bernoulli sampling.
\end{minipage}\hfill
\ifisarxiv
\begin{minipage}{0.39\textwidth}
\else
\begin{minipage}{0.43\textwidth}
\fi
\renewcommand{\thealgorithm}{}
\floatname{algorithm}{}
\vspace{-7mm}
\begin{algorithm}[H]
{\fontsize{9}{9}\selectfont
  \caption{\bf \small Sampling \ $S\sim\DPPreg{p}(\X,\A)$}
  \begin{algorithmic}[0]
    \STATE \textbf{Input:} $\X\!\in\!\R^{n\times d}\!,\text{ psd } \A\!\in\!\R^{d\times
      d}\!, p\!\in\![0,1]^n$\\[1mm]
    \STATE Compute $\Z \leftarrow \A+\X^\top\D_p\X$\\[1mm]
    \STATE Compute \textsc{SVD} of  $\B=\D_p^{\sfrac12}\X\Z^{-\sfrac12}$\\[1mm]
    \STATE Sample $T\sim\DPPcor(\B\B^\top)$ \hfill \cite{dpp-independence}\\[1mm]
    \STATE Sample $b_i\sim\mathrm{Bernoulli}(p_i)$ for $i\in[n]$\\[1mm]
    \RETURN $S=T\cup\{i:b_i=1\}$
 \end{algorithmic}
}
\end{algorithm}
\end{minipage}

\vspace{2mm}
\begin{lemma}\label{t:reduction}
  Given $\X$, $\A$ and $\D_p$ as in Theorem \ref{t:algorithm}, assume that $\A$ and $\I-\D_p$ are
  invertible. Then,
  \begin{align*}
    \DPPreg{p}(\X,\A)=\DPPens\big(\Dbt+
    \Dbt^{\sfrac12}\X\A^{-1}\X^\top
    \Dbt^{\sfrac12}\big),\quad\text{where}
    \quad\Dbt = \D_p(\I-\D_p)^{-1}.
  \end{align*}
\end{lemma}
\begin{proof}
  Let $S\sim \DPPreg{p}(\X,\A)$. By Definition \ref{d:r-dpp} and
  the fact that $\det(\A\B+\I)=\det(\B\A+\I)$,
  \begin{align*}
    \Pr(S) &\propto \det(\X_S^\top\X_S+\A)\cdot
\prod_{i\in S}p_i\cdot\prod_{i\not\in S}(1-p_i) = \det(\X_S^\top\X_S+\A)\cdot
\prod_{i\in S}\frac{p_i}{1-p_i}\cdot\prod_{i=1}^n(1-p_i)
\\ &\propto\det\!\big(\A(\A^{-1}\X_S^\top\X_S+\I)\big) \det(\Dbt_{S,S})
     =\det(\A)\det(\A^{-1}\X_S^\top\X_S+\I) \det(\Dbt_{S,S})
\\ &\propto \det(\X_S \A^{-1}\X_S^\top+\I) \det(\Dbt_{S,S})
=\det\!\Big(\big[\Dbt^{\sfrac12}\X \A^{-1}\X^\top\Dbt^{\sfrac12}+\Dbt\big]_{S,S}\Big),
  \end{align*}
which matches the definition of the L-ensemble DPP.
\end{proof}
 At this point, to sample from  $\DPPreg{p}(\X,\A)$, we could simply
invoke any algorithm for sampling from
an ensemble DPP. However, this would only work for invertible
$\A$, which in particular excludes the important case of
$\A=\zero$ corresponding to classical experimental
design. Moreover, the standard algorithm would require computing the
eigendecomposition of the ensemble kernel, which (at
least if done na\"ively) costs $O(n^3)$. Even after this is done, the
sampling cost would still be $O(n\,|S|^2)$ which can be considerably
more than $O(nd^2)$. We first address the issue of invertibility of matrix
$\A$ by expressing our distribution via a correlation DPP.
\begin{lemma}\label{l:correlation}
  Given $\X$, $\A$, and $\D_p$ as in Theorem \ref{t:algorithm} (without
  any additional assumptions), we have
  \begin{align*}
    \DPPreg{p}(\X,\A)= \DPPcor\big(\D_p +
    (\I-\D_p)^{\sfrac12}\,\D_p^{\sfrac12}\X(\A+\X^\top\D_p\X)^{-1}\X^\top\D_p^{\sfrac12}(\I-\D_p)^{\sfrac12}\big).
    \end{align*}
  \end{lemma}
  When $\A$ and $\I-\D_p$ are invertible, then the proof (given in
  Appendix \ref{a:proofs}) is a straightforward calculation.
Then, we use a limit argument with $p_\epsilon=(1-\epsilon)p$ and
$\A_\epsilon=\A+\epsilon\I$, where $\epsilon\rightarrow 0$.

Finally, we show that the correlation DPP arrived at in Lemma
\ref{l:correlation} can be decomposed into a smaller DPP plus
Bernoulli sampling. In fact, in the following lemma we obtain a more
general recipe for combining DPPs with Bernoulli sampling, which may
be of independent interest. Note that if $b_i\sim\mathrm{Bernoulli}(p_i)$ are independent random
variables then $\{i:b_i\!=1\!\}\sim\DPPcor(\D_p)$.
\begin{lemma}\label{l:decomposition}
  Let $\K$ and $\D$ be $n\times n$ psd matrices with eigenvalues between
0 and 1, and assume that $\D$ is diagonal. If\, $T\!\sim\DPPcor(\K)$ and
$R\sim\DPPcor(\D)$, then
\begin{align*}T\cup R\sim \DPPcor\big(\D+(\I-\D)^{\sfrac12}\K
  (\I-\D)^{\sfrac12}\big).
  \end{align*}
\end{lemma}
The lemma is proven in Appendix \ref{a:proofs}. Theorem
\ref{t:algorithm} now follows by combining Lemmas \ref{l:correlation} and \ref{l:decomposition}.

\section{Guarantees for Bayesian experimental design}
\label{s:guarantees}
In this section we prove our main results regarding Bayesian
experimental design (Theorems \ref{t:q1} and \ref{t:q2}). First, we
establish certain properties of the regularized DPP distribution that
make it effective in this setting. Even though the size of the
sampled subset $S\sim\DPPreg{p}(\X,\A)$ is random and can be as large
as $n$, it is also highly concentrated around its expectation, which
can be bounded in terms of the $\A$-effective dimension. This is
crucial, since both of our main results require a subset of
deterministically bounded size. Recall that the effective dimension is
defined as a function $d_{\A}(\Sigmab) =
\tr\big(\Sigmab(\A+\Sigmab)^{-1}\big)$. The omitted proofs are in
Appendix \ref{a:proofs}.
\begin{lemma}\label{l:size}
  Given any $\X\in\R^{n\times d}$, $p\in[0,1]^n$ and a psd matrix $\A$
s.t.~$\sum_ip_i\x_i\x_i^\top\!+\A$ is
full rank, let $S=T\cup \{i:b_i=1\}\sim \DPPreg{p}(\X,\A)$ be defined
as in Theorem \ref{t:algorithm}. Then
\begin{align*}
\E\big[|S|\big]\leq \E\big[|T|\big] + \E\Big[\sum_i b_i\Big] = d_{\A}\Big(\sum_ip_i\x_i\x_i^\top\Big) + \sum_i
  p_i.
\end{align*}
\end{lemma}
Next, we show two expectation inequalities for the matrix inverse and
matrix determinant, which hold for the regularized DPP.
We use them to bound the Bayesian optimality criteria in expectation.
\begin{lemma}\label{t:expectations}
Whenever $S\sim \DPPreg{p}(\X,\A)$ is a well-defined distribution it holds that
\begin{align}
  \E\Big[\big(\X_S^\top\X_S+\A\big)^{-1}\Big]
  &\preceq \Big(
    \sum_ip_i\x_i\x_i^\top +\A \Big)^{-1},\label{eq:sqinv}
  \\
\E\Big[\det\!\big(\X_S^\top\X_S+\A\big)^{-1}\Big]
  &\leq\det\!\Big(
    \sum_ip_i\x_i\x_i^\top +\A \Big)^{-1}.\label{eq:det}
\end{align}
\end{lemma}
\begin{corollary}
Let $f_{\A}$ be A/C/D/V-optimality. Whenever $S\sim\DPPreg{p}(\X,\A)$ is
well-defined,
\begin{align*}
  \E\big[f_{\A}(\X_S^\top\X_S)\big] \leq f_{\A}\Big(\sum_i p_i\x_i\x_i^\top\Big).
\end{align*}
\end{corollary}
\begin{proof}
  In the case of A-, C-, and V-optimality, the function $f_{\A}$ is a
  linear transformation of the matrix $(\X_S^\top\X_S+\A)^{-1}$ so the
  bound follows from \eqref{eq:sqinv}. For D-optimality, we apply
  \eqref{eq:det} as follows:
  \begin{align*}
    \E\big[f_{\A}(\X_S^\top\X_S)\big]
    & = \E\Big[\det\!\big(\X_S^\top\X_S+\A\big)^{-\sfrac1d}\Big]
\leq
     \E\bigg[\Big(\det\!\big(\X_S^\top\X_S+\A\big)^{-\sfrac1d}\Big)^d\bigg]^{\sfrac1d}
    \\ &=  \E\Big[\!\det\!\big(\X_S^\top\X_S+\A\big)^{-1}\Big]^{\sfrac1d}=
         \det\!\Big(\sum_i p_i\x_i\x_i^\top\Big)^{-\sfrac1d},
  \end{align*}
  which completes the proof.
\end{proof}
Finally, we present the key lemma that puts everything together. This
result is essentially a generalization of Theorem \ref{t:q1} from
which also follows Theorem \ref{t:q2}.
\begin{lemma}\label{l:guarantees}
  Let $f_{\A}$ be A/C/D/V-optimality and $\X$ be $n\times
  d$.  For some $w=(w_1,\dots,w_n)\in[0,1]^n$, let
  $\Sigmab_w=\sum_iw_i\x_i\x_i^\top$ and assume that
  $\sum_i w_i=k\in[n]$. If $k\geq 4\,d_{\A}(\Sigmab_w)$, then a
subset $S\subseteq[n]$ of size $k$ can be found in $O(ndk+k^2d^2)$ time that satisfies
  \begin{align*}
    f_{\A}\big(\X_S^\top\X_S\big) \leq \bigg(\ 1+8\,\frac{d_{\A}(\Sigmab_w)}{k} + 8\sqrtshort{\frac{\ln (k/d_{\A}(\Sigmab_w))}{k}}\ \bigg)\cdot f_{\A}\big(\Sigmab_w\big).
  \end{align*}
  \end{lemma}
\begin{proof}
  Let $p=(p_1,\dots,p_n)$ be defined so that $p_i =
  \frac{w_i}{1+\epsilon}$, and suppose that $S\sim
  \DPPreg{p}(\X,\A)$. Then, using Theorem \ref{t:expectations}, we have
  \begin{align*}
\Pr\big(|S|\leq k\big)\,\E\Big[f_{\A}(\X_S^\top\X_S)\mid |S|\leq k\Big]&\leq
\E\big[f_{\A}(\X_S^\top\X_S)\big]\leq
f_{\A}\Big(\sum_ip_i\x_i\x_i^\top\Big)
\\ &\leq (1+\epsilon)\cdot f_{\A}\Big(\sum_iw_i\x_i\x_i^\top\Big).
  \end{align*}
  Using Lemma \ref{l:size} we can bound the expected size of $S$ as
follows:
\begin{align*}
  \E\big[|S|\big]
  &\leq d_{\A}(\Sigmab_w) + \sum_i p_i
=d_{\A}(\Sigmab_w)+\frac{k}{1+\epsilon}=k\cdot \Big(1+\frac {d_{\A}(\Sigmab_w) }k - \frac{\epsilon}{1+\epsilon}\Big).
\end{align*}
Let $d_w = d_{\A}(\Sigmab_w)$ and $\alpha=1+\frac {d_w} k -\frac{\epsilon}{1+\epsilon}$.
If $1\geq\epsilon\geq \frac{4d_w}{k}$, then $\alpha\leq 1 +
\frac\epsilon4-\frac\epsilon2=1-\frac\epsilon4$. Since $\DPPreg{p}(\X,\A)$
is a  determinantal point process, $|S|$ is a
Poisson binomial r.v.~so for $\epsilon\geq 6\sqrt{\frac{\ln(k/d_w)}{k}}$,
\begin{align*}
  \Pr(|S|>k) \leq \ee^{-\frac{(k-\alpha k)^2}{2k}} = \ee^{-\frac
  k2(1-\alpha)^2}\leq \ee^{-\frac{k\epsilon^2}{32}}\leq \frac{d_w}k.
\end{align*}
For any $\epsilon\geq
4\,\frac{d_w}k + 6\sqrt{\frac{\ln(k/d_w)}{k}}$, we have
\begin{align*}
\E\big[f_{\A}(\X_S^\top\X_S)\mid |S|\leq k\big] &\leq
  \frac{1+\epsilon}{1-\frac{d_w}k}\cdot f_{\A}(\Sigmab_w)\leq
  \bigg(1+\frac{\epsilon+\frac{d_w}k}{1-\frac{d_w}k}\bigg)\cdot
f_{\A}(\Sigmab_w)
\\ &\leq \bigg(1+7\,\frac{d_w}{k} +
  8\sqrtshort{\frac{\ln(k/d_w)}{k}}\bigg)\cdot f_{\A}(\Sigmab_w).
\end{align*}
Denoting $\E\big[f_{\A}(\X_S^\top\X_S)\mid |S|\leq k\big]$ as $F_k$,
Markov's inequality implies that
\[\Pr\Big(f_{\A}(\X_S^\top\X_S) \geq
  (1+\delta) F_k\ \mid\ |S|\leq k\Big)\leq \frac1{1+\delta}.\]
Also, we showed
that $\Pr(|S|\leq k)\geq 1-\frac{d_w}{k}\geq \frac34$. Setting $\delta=\frac{d_w}{C
  k}$ for sufficiently large $C$ we obtain that with probability
$\Omega(\frac{d_w}{k})$, the random set $S$ has size at most $k$ and
\begin{align*}
f_{\A}(\X_S^\top\X_S)
&\leq \bigg(1+\frac{d_w}{C k}\bigg)\cdot \bigg(1+7\,\frac{d_w}{k} +
  8\sqrtshort{\frac{\ln(k/d_w)}{k}}\bigg)\cdot f_{\A}(\Sigmab_w)\\
&\leq
\bigg(1+8\,\frac{d_w}{k} +
  8\sqrtshort{\frac{\ln(k/d_w)}{k}}\bigg)\cdot f_{\A}(\Sigmab_w).
\end{align*}
We can sample from $\DPPreg{p}(\X,\A)$ conditioned on
$|S|\leq k$ and $f_{\A}(\X_S^\top\X_S)$ bounded as above by rejection
sampling. When $|S|<k$,  the set is completed to
$k$ with arbitrary indices. On average, $O(\frac k{d_w})$ samples from
$\DPPreg{p}(\X,\A)$ are needed, so the cost is $O(nd^2)$ for the
eigendecomposition, $O(\frac{k}{d_w}\cdot nd_w^2)=O(nd_wk)$ for
sampling and $O(\frac k{d_w}\cdot kd^2)$ for recomputing $f_{\A}(\X_S^\top\X_S)$.
\end{proof}
To prove the main results, we use Lemma \ref{l:guarantees} with
appropriately chosen weights $w$.
\begin{proofof}{Theorem}{\ref{t:q1}}
Let $w=(\frac kn,...,\frac kn)$ in Lemma \ref{l:guarantees}. Then, we have
$\Sigmab_w=\frac kn\Sigmab_{\X}$ and also $d_{\A}(\Sigmab_w) =
d_{\frac nk\!\A}$. Since for any set $S$ of size $k$, we
have $\opt\leq f_{\A}(\X_S^\top\X_S)$, the result follows.
\end{proofof}
\begin{proofof}{Theorem}{\ref{t:q2}}
  As discussed in \cite{near-optimal-design,boyd2004convex}, the following convex
relaxation of experimental design can be written as a semi-definite
program and solved using standard SDP solvers:
\begin{align}
w^* \ = \  \argmin_w\ \
  f_{\A}\Big(\sum_{i=1}^nw_i\x_i\x_i^\top\Big),\quad\text{subject to}\quad
  \forall_i\ \ 0\leq w_i\leq 1,\quad \sum_i w_i=k.\label{eq:sdp}
\end{align}
The solution $w^*$ satisfies $f_{\A}\big(\Sigmab_{w^*}\big)\leq
\opt$. If we use $w^*$ in Lemma \ref{l:guarantees}, then observing
that $d_{\A}(\Sigmab_{w^*})\leq d_{\A}$, and setting $k\geq C(\frac
{d_{\A}}{\epsilon} + \frac{\log1/\epsilon}{\epsilon^2})$ for
sufficiently large $C$, the algorithm in the lemma finds subset $S$
such that $f_{\A}(\X_S^\top\X_S)\leq (1+\epsilon)\cdot
f_{\A}(\Sigmab_w)\leq
(1+\epsilon)\cdot\opt$.
Note that we did not need to
solve the SDP exactly, so approximate solvers could be used instead.
\end{proofof}

\input{bayesian_experiments.tex}

\subsubsection*{Acknowledgements}
MWM would like to acknowledge ARO, DARPA, NSF and ONR for providing partial
  support of this work. Also, MWM and MD thank the NSF for
  funding via the NSF TRIPODS program. The authors thank Uthaipon
  T.~Tantipongpipat for valuable discussions.

\bibliographystyle{alpha}
\bibliography{pap}

\newpage
\appendix

\section{Properties of regularized DPPs}
\label{a:proofs}
In this section we provide proofs omitted from Sections \ref{s:r-dpp}
and \ref{s:guarantees}. We start with showing the fact that the
regularized DPP distribution $\DPPreg{p}(\X,\A)$ is a correlation DPP.
\begin{lemma}[restated Lemma \ref{l:correlation}]
  Given $\X$, $\A$, and $\D_p$ as in Theorem \ref{t:algorithm}, we have
  \begin{align*}
    \DPPreg{p}(\X,\A)= \DPPcor\big(\D_p +
    (\I-\D_p)^{\sfrac12}\,\D_p^{\sfrac12}\X(\A+\X^\top\D_p\X)^{-1}\X^\top\D_p^{\sfrac12}(\I-\D_p)^{\sfrac12}\big).
    \end{align*}
  \end{lemma}
  \begin{proof}
    First, we show this under the invertibility assumptions of Lemma
    \ref{t:reduction}, i.e., given that $\A$ and $\I-\D_p$ are
    invertible. In this case $\DPPreg{p}(\X,\A)=\DPPens(\L)$, where
    \begin{align}\L=\Dbt+
    \Dbt^{\sfrac12}\X\A^{-1}\X^\top
      \Dbt^{\sfrac12}\quad\text{ and }\quad\Dbt=\D_p(\I-\D_p)^{-1}.\label{eq:ensemble}
\end{align}
Converting this to
    a correlation kernel $\K$ and denoting $\Xt=\D_p^{\sfrac12}\X$, we obtain
    \begin{align*}
      \K&=\I-(\I+\L)^{-1}
\\      &=\I - \big(\I + (\I-\D_p)^{-1}\D_p +
(\I-\D_p)^{-\sfrac12}\Xt\A^{-1}\Xt^\top(\I-\D_p)^{-\sfrac12}\big)^{-1}
\\ &= \I - \big((\I-\D_p)^{-1} +
     (\I-\D_p)^{-\sfrac12}\Xt\A^{-1}\Xt^\top(\I-\D_p)^{-\sfrac12}\big)^{-1}
\\ &=\I -
     (\I-\D_p)^{\sfrac12}(\I+\Xt\A^{-1}\Xt^\top)^{-1}(\I-\D_p)^{\sfrac12}
      \\
      &\overset{(*)}{=}\I -
     (\I-\D_p)^{\sfrac12}\big(\I-\Xt\A^{-\sfrac12}
     (\I+\A^{-\sfrac12}\Xt^\top\Xt\A^{-\sfrac12})^{-1}
     \A^{-\sfrac12}\Xt^\top\big)(\I-\D_p)^{\sfrac12}
      \\
      &=\I - (\I-\D_p) +
     (\I-\D_p)^{\sfrac12}\Xt(\A+\Xt^\top\Xt)^{-1}\Xt^\top(\I-\D_p)^{\sfrac12}
\\ &=\D_p + (\I-\D_p)^{\sfrac12}\Xt(\A+\Xt^\top\Xt)^{-1}\Xt^\top(\I-\D_p)^{\sfrac12},
    \end{align*}
    where $(*)$ follows from Fact 2.16.19 in
    \cite{matrix-mathematics}. Note that converting from $\L$ to $\K$
    got rid of the inverses $\A^{-1}$  and $(\I-\D_p)^{-1}$ appearing
    in \eqref{eq:ensemble}. The intuition 
    is that when $\A$ or $\I-\D_p$ is non-invertible, then
    $\DPPreg{p}(\X,\A)$ is not an L-ensemble but it is still a
    correlation DPP. To show this, we use a limit argument. For
    $\epsilon\in[0,1]$, let $p_\epsilon=(1-\epsilon)p$ and
    $\A_\epsilon=\A+\epsilon\I$. Observe that if $\epsilon>0$ then $\A_\epsilon$ and
    $\I-\D_{p_{\epsilon}}$ are always invertible even if $\A$ and
    $\I-\D_p$ are not. Denote $\K_\epsilon$ as the
    above correlation kernel with $p$ replaced by $p_{\epsilon}$ and
    $\A$ replaced by $\A_{\epsilon}$. Note that all matrix operations
    defining kernel $\K_\epsilon$ are continuous w.r.t. $\epsilon\in[0,1]$, including the inverse, since
    $\A+\Xt^\top\Xt$ is assumed to be invertible. Therefore, the
    following equalities hold (with limits taken point-wise and $\epsilon>0$):
    \begin{align*}
      \DPPreg{p}(\X,\A) = \lim_{\epsilon\rightarrow
      0}\DPPreg{p_\epsilon}(\X,\A_{\epsilon}) =
      \lim_{\epsilon\rightarrow 0}\DPPcor(\K_\epsilon) = \DPPcor(\K),
    \end{align*}
where we did not have to assume invertibility of $\A$ or $\I-\D_p$.
\end{proof}
We now prove a lemma about combining a determinantal point process
with Bernoulli sampling, which itself is a DPP with a diagonal correlation
kernel.
\begin{lemma}[restated Lemma \ref{l:decomposition}]
  Let $\K$ and $\D$ be $n\times n$ psd matrices with eigenvalues between
0 and 1, and assume that $\D$ is diagonal. If\, $T\!\sim\DPPcor(\K)$ and
$R\sim\DPPcor(\D)$, then
\begin{align*}T\cup R\sim \DPPcor\big(\D+(\I-\D)^{\sfrac12}\K
  (\I-\D)^{\sfrac12}\big).
  \end{align*}
\end{lemma}
\begin{proof}
For this proof we will use the shorthand $\K_A$ for $\K_{A,A}$. If
$\D$ has no zeros on the diagonal then $\det(\D_A)>0$ for all
$A\subseteq[n]$ and
  \begin{align*}
    \Pr(A \subset T \cup R)
    &=\sum_{B\subset A}\Pr(R\cap A=A\setminus B) \ \Pr(B\subseteq T)\\
    &=\sum_{B\subset A}\det(\D_{A\setminus
      B})\det\!\big([\I-\D]_B\big)\,\det(\K_B) \\
    &=\sum_{B\subset A}\det(\D_{A\setminus
      B})\det\!\Big(\big[(\I-\D)^{\sfrac12}\K(\I-\D)^{\sfrac12}\big]_B\Big) \\
    &=\det(\D_A) \sum_{B\subset A}\det\!\Big(\big[\D^{-\sfrac12}
      (\I-\D)^{\sfrac12}\K(\I-\D)^{\sfrac12}\D^{-\sfrac12}\big]_{B}\Big)\\[-1mm]
    &\overset{(*)}{=}\det(\D_A)\det\!\Big(\I + \big[\D^{-\sfrac12}
      (\I-\D)^{\sfrac12}\K(\I-\D)^{\sfrac12}\D^{-\sfrac12}\big]_A\Big)\\
    &=\det\!\Big(\big[\D+ (\I-\D)^{\sfrac12}\K(\I-\D)^{\sfrac12}\big]_A\Big),
  \end{align*}
  where $(*)$ follows from a standard determinantal identity used to
  compute the L-ensemble partition function
  \cite[Theorem~2.1]{dpp-ml}. If $\D$ has zeros on the diagonal, a
  similar limit argument as in Lemma \ref{l:correlation} with
  $\D_\epsilon=\D+\epsilon\,\I$ holds.
\end{proof}
Next, we give a bound on the expected size of a regularized DPP.
\begin{lemma}[restated Lemma \ref{l:size}]
  Given any $\X\in\R^{n\times d}$, $p\in[0,1]^n$ and a psd matrix $\A$
  s.t.~$\sum_ip_i\x_i\x_i^\top\!+\A$ is
  full rank, let $S=T\cup \{i:b_i=1\}\sim \DPPreg{p}(\X,\A)$ be defined
as in Theorem \ref{t:algorithm}. Then
\begin{align*}
\E\big[|S|\big]\leq \E\big[|T|\big] + \E\Big[\sum_i b_i\Big] =
  d_{\A}\Big(\sum_ip_i\x_i\x_i^\top\Big) + \sum_i  p_i.
\end{align*}
\end{lemma}
\begin{proof}
For correlation kernels it is known that the expected size of $\DPPcor(\K)$
is $\tr(\K)$. Thus, using $\D_p=\diag(p)$, we can invoke Lemma \ref{l:correlation} to obtain
\begin{align*}
  \E\big[|S|\big]
  &= \tr\big(\D_p +
  (\I-\D_p)^{\sfrac12}\,\D_p^{\sfrac12}\X(\A+\X^\top\D_p\X)^{-1}
                    \X^\top\D_p^{\sfrac12}(\I-\D_p)^{\sfrac12}\big)
  \\ &\leq \tr(\D_p) +\tr\big(\D_p^{\sfrac12}
       \X(\A+\X^\top\D_p\X)^{-1}\X^\top\D_p^{\sfrac12}\big)
\\ &=\tr(\D_p) +\tr\big(\X^\top\D_p\X (\A+\X^\top\D_p\X)^{-1}\big) =
     \tr(\D_p) + d_{\A}(\X^\top\D_p\X),
\end{align*}
from which the claim follows.
\end{proof}
Next, we show two expectation inequalities for the matrix inverse and
matrix determinant.
\begin{lemma}[restated Lemma \ref{t:expectations}]
Whenever $S\sim \DPPreg{p}(\X,\A)$ is a well-defined distribution it holds that
\begin{align}
  \E\Big[\big(\X_S^\top\X_S+\A\big)^{-1}\Big]
  &\preceq \Big(
    \sum_ip_i\x_i\x_i^\top +\A \Big)^{-1},\label{eq:sqinv2}
  \\
\E\Big[\det\!\big(\X_S^\top\X_S+\A\big)^{-1}\Big]
  &\leq\det\!\Big(
    \sum_ip_i\x_i\x_i^\top +\A \Big)^{-1}.\label{eq:det2}
\end{align}
\end{lemma}
\begin{proof}
For a square matrix $\M$, define its adjugate, denoted $\adj(\M)$, as a matrix
whose $i,j$-th entry is $(-1)^{i+j}\det(\M_{-j,-i})$, where
$\M_{-j,-i}$ is the matrix $\M$ without $j$th row and $i$th column. If
$\M$ is invertible, then $\adj(\M) = \det(\M)\M^{-1}$. Now, let
$b_i\sim\mathrm{Bernoulli}(p_i)$ be independent random variables. As seen
in previous section, the identity
$\E[\det(\sum_ib_i\x_i\x_i^\top+\A)]=\det(\sum_ip_i\x_i\x_i^\top+\A)$ gives
us the normalization constant for $\DPPreg{p}(\X,\A)$. Moreover, as
noted in a different context by \cite{determinantal-averaging}, when
applied entrywise to the adjugate matrix, this identity implies that
$\E[\adj(\sum_ib_i\x_i\x_i^\top+\A)]=\adj(\sum_ip_i\x_i\x_i^\top+\A)$. Let
$\Ic$ denote the set of all subsets $S\subseteq [n]$ such that
$\X_S^\top\X_S+\A$ is invertible. We have
\begin{align*}
  \E\Big[\big(\X_S^\top\X_S+\A\big)^{-1}\Big]
  &= \sum_{S\in\Ic}\big(\X_S^\top\X_S+\A\big)^{-1}
\frac{\det(\X_S^\top\X_S+\A)}{\det(\sum_ip_i\x_i\x_i^\top+\A)}\ \prod_{i\in
    S}p_i\prod_{i\not\in S}(1-p_i)
\\ &=\sum_{S\in\Ic}\frac{\adj(\X_S^\top\X_S+\A\big)}
{\det(\sum_ip_i\x_i\x_i^\top+\A)}\ \prod_{i\in
     S}p_i\prod_{i\not\in S}(1-p_i)
\\ &\preceq \sum_{S\subseteq[n]}\frac{\adj(\X_S^\top\X_S+\A\big)}
{\det(\sum_ip_i\x_i\x_i^\top+\A)}\ \prod_{i\in
     S}p_i\prod_{i\not\in S}(1-p_i)
\\ &=\frac{\E\big[\!\adj(\sum_i
     b_i\x_i\x_i^\top+\A)\big]}{\det(\sum_ip_i\x_i\x_i^\top+\A)}
  \\
  &=\frac{\adj(\sum_ip_i\x_i\x_i^\top+\A)}{\det(\sum_ip_i\x_i\x_i^\top+\A)}
    = \Big(\sum_ip_i\x_i\x_i^\top+\A\Big)^{-1}.
\end{align*}
Note that if $\Ic$ contains all subsets of $[n]$, for example when
$\A\succ\zero$, then the inequality turns into equality. Thus, we
showed \eqref{eq:sqinv2}, and \eqref{eq:det2} follows even more easily:
\begin{align*}
  \E\Big[\!\det\!\big(\X_S^\top\X_S+\A\big)^{-1}\Big]
  &= \sum_{S\in\Ic}
\frac{1}{\det(\sum_ip_i\x_i\x_i^\top+\A)}\ \prod_{i\in
    S}p_i\prod_{i\not\in S}(1-p_i)
\leq \det\!\Big(\sum_i p_i\x_i\x_i^\top\Big)^{\!-1}\!,
\end{align*}
where the equality holds if $\Ic$ consists of all subsets of $[n]$.
\end{proof}

\section{Comparison of different effective dimensions}
\label{a:deff}
In this section we compare the two notions of effective dimension for
Bayesian experimental design considered in this work. Here, we let
$\X$ be the full $n\times d$ design matrix and use $k$ to denote the
desired subset size. Recall that the effective dimension is  defined
as a function of the data covariance matrix $\Sigmab_\X=\X^\top\X$ and
the prior precision matrix $\A$:
It is given by $d_{\A} =
\tr\big(\Sigmab_\X(\A+\Sigmab_\X)^{-1}\big)$. In
\cite{regularized-volume-sampling} it was suggested that $d_{\A}$
should also be used as the effective dimension for the experimental
design problem. Our results suggest it may not reflect the true degrees
of freedom of the problem because it does not scale with subset size
$k$. Instead we propose to use the \emph{scaled effective dimension}
$d_{\frac nk\!\A}$. Thus, the two definitions we are comparing can be
summarized as follows:
\begin{description}
  \item[Full effective dimension]\quad\ \ \,$d_{\A} =
    \tr\big(\Sigmab_\X(\A+\Sigmab_\X)^{-1}\big)$,
    \item[Scaled effective dimension] $d_{\frac nk\!\A} =
      \tr\big(\Sigmab_\X(\tfrac nk\A + \Sigmab_\X)^{-1}\big)$.
    \end{description}
    Here, we demonstrate that these two effective dimensions can be
    very different for some matrices and quite similar on others. For
    simplicity, we consider two diagonal data covariance matrices as
    our examples: \emph{identity covariance}, $\Sigmab_1 = \I$, and an
    \emph{approximately low-rank covariance}, $\Sigmab_2 =
    (1-\epsilon)\frac ds \I_S+\epsilon\I$, where $\I_S$ is the
    diagonal matrix with ones on the entries indexed by subset
    $S\subseteq [d]$ of
    size $s<d$ and zeros
    everywhere else. The second matrix is scaled in such way so that
    $\tr(\Sigmab_1)=\tr(\Sigmab_2)$. We use $d=100$, $s=10$ and
    $\epsilon=10^{-2}$. The prior precision matrix is
    $\A=10^{-2}\,\I$. Figure \ref{f:deff} plots the scaled effective
    dimension $d_{\frac nk\!\A}$ as a function of $k$, against the
    full effective dimension for both examples. Unsurprisingly, for
    the identity covariance the full effective dimension is almost
    $d$, and the scaled effective dimension goes up very quickly to
    match it. On the other hand, for the approximately low-rank
    covariance, $d_{\A}\approx 55$ is considerably less then
    $d=100$. Interestingly, the gap between the $d_{\frac nk\!\A}$ and 
    $d_{\A}$ for moderately small values of $k$ is even bigger. Our
    theory suggests that $d_{\frac nk\!\A}$ is a valid indicator of
    Bayesian degrees of freedom when $k\geq C\cdot d_{\frac nk\!\A}$
    for some small constant $C$ (Theorem \ref{t:q1} has $C=4$, but we
    believe this can be improved to $1$). While for the identity
    covariance the condition $k\geq
    d_{\frac nk\!\A}$ is almost equivalent to $k\geq d_{\A}$, in the
    approximately low-rank case, $k\geq d_{\frac nk\!\A}$ holds for
    $k$ as small as 20, much less than $d_{\A}$.
      \begin{figure}[htbp]
    \centering
    \includegraphics[width=0.5\textwidth]{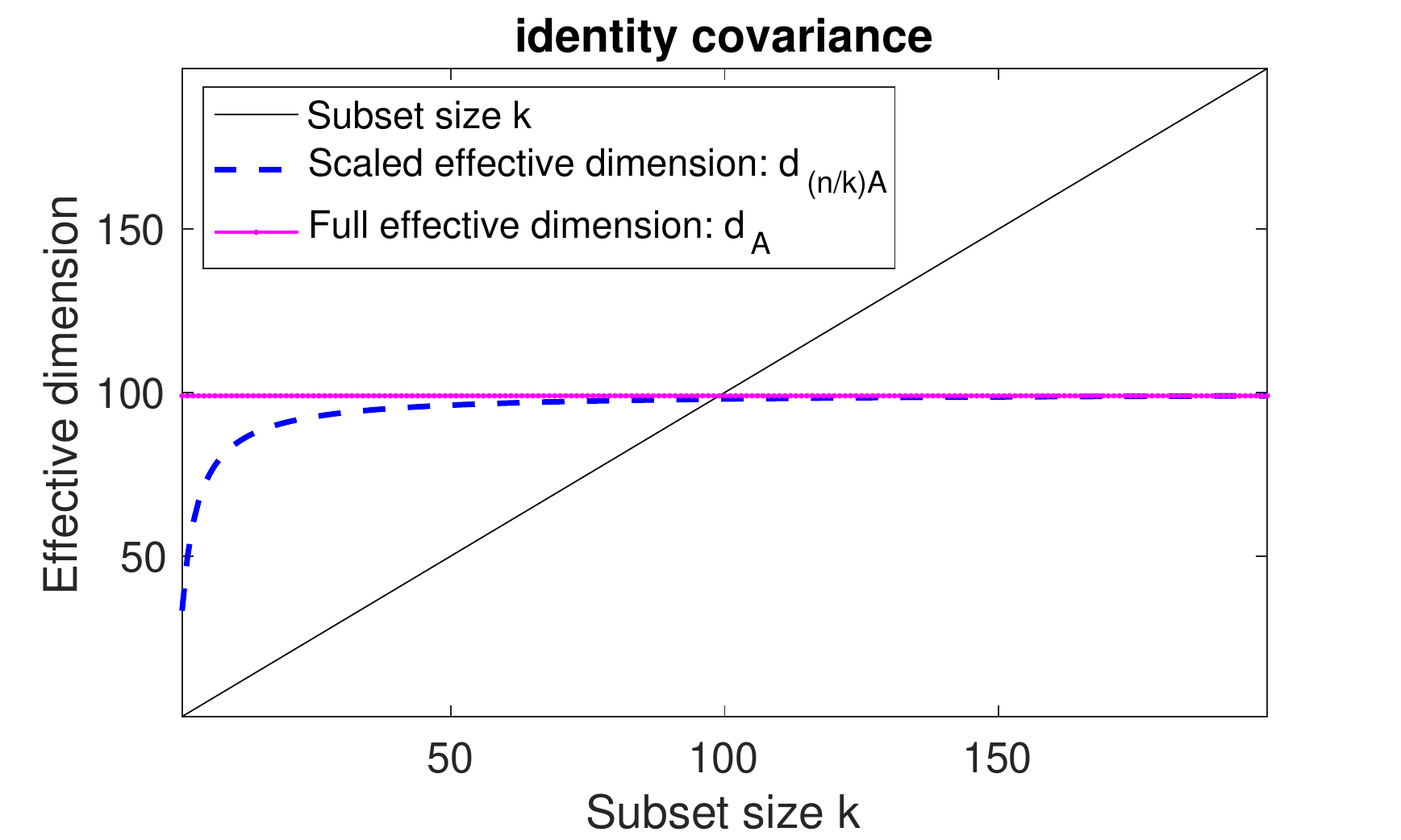}\nobreak\includegraphics[width=0.5\textwidth]{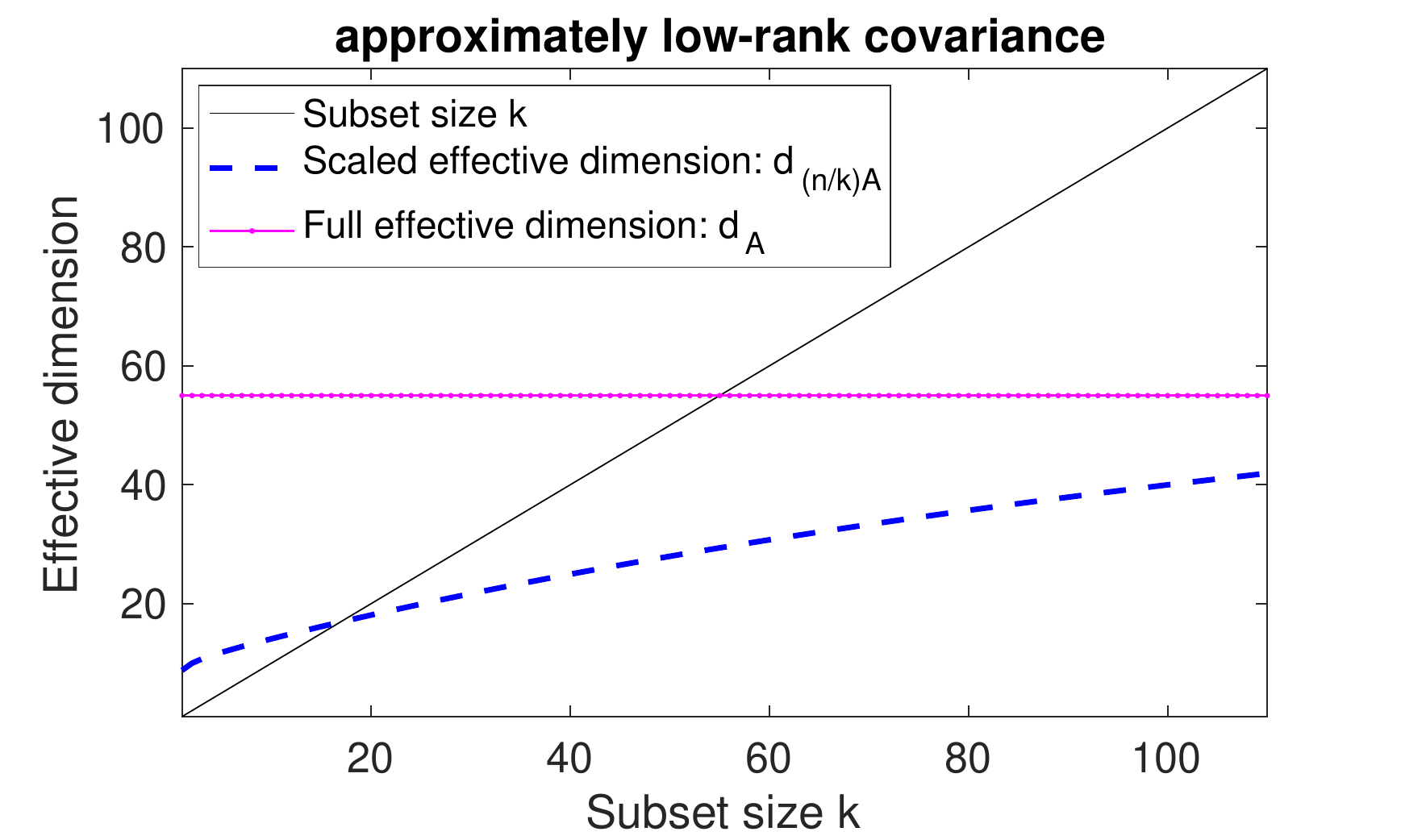}
    \caption{Scaled effective dimension compared to the full effective
      dimension for two diagonal data covariance matrices, with
      $\A=10^{-2}\,\I$.}
    \label{f:deff}
\end{figure}
\section{Additional details for the experiments}
\label{a:experiments}
\input{bayesian_experiments-appendix.tex}

\end{document}

%% file: shortdefs.tex
\def\opt{{\textsc{OPT}_k}}

\def\cmark{\Green{\checkmark}}
\def\xmark{\Red{\large\sffamily x}}

\newif\ifDRAFT
\DRAFTtrue
\ifDRAFT
\newcommand{\marrow}{\marginpar[\hfill$\longrightarrow$]{$\longleftarrow$}}
\newcommand{\niceremark}[3]
   {\textcolor{red}{\textsc{#1 #2:} \marrow\textsf{#3}}}
\newcommand{\ken}[2][says]{\niceremark{Ken}{#1}{#2}}

\newcommand{\michael}[2][says]{\niceremark{Michael}{#1}{#2}}
\newcommand{\michal}[2][says]{\niceremark{Michal}{#1}{#2}}
\newcommand{\feynman}[2][says]{\niceremark{Feynman}{#1}{#2}}
\else
\newcommand{\ken}[1]{}
\newcommand{\michael}[1]{}
\newcommand{\michal}[1]{}
\newcommand{\feynman}[1]{}
\fi

\def\ee{\mathrm{e}}

\def\poly{{\mathrm{poly}}}
\def\DPP{{\mathrm{DPP}}}
\def\DPPcor{{\DPP_{\!\mathrm{cor}}}}
\def\DPPens{{\DPP_{\!\mathrm{ens}}}}
\newcommand{\DPPreg}[1]{{\DPP_{\!\mathrm{reg}}^{#1}}}

\def\checkmark{\tikz\fill[scale=0.4](0,.35) -- (.25,0) --
(1,.7) -- (.25,.15) -- cycle;}

\newcommand{\sqrtshort}[1]{{\sqrt{\white{\Big|}\!\!\smash{\text{\fontsize{9}{9}\selectfont$#1$}}}}}
\newenvironment{proofof}[2]{\par\vspace{2mm}\noindent\textbf{Proof of {#1} {#2}}\ }{\hfill\BlackBox}

\DeclareMathOperator{\adj}{\textnormal{adj}}

\def\Sigmab{\mathbf{\Sigma}}

\def\cb {\mathbf{c}}

\def\Xt{\widetilde{X}}
\def\Dbt{\widetilde{\D}}

\def\L{\mathbf{L}}

\def\Nc{\mathcal{N}}

\def\K{\mathbf K}

\ifx\BlackBox\undefined
\newcommand{\BlackBox}{\rule{1.5ex}{1.5ex}}  
\fi
\DeclareMathOperator*{\argmin}{\mathop{\mathrm{argmin}}}

\DeclareMathOperator*{\diag}{\mathop{\mathrm{diag}}}
\def\x{\mathbf x}
\def\y{\mathbf y}

\def\w{\mathbf w}

\def\zero{\mathbf 0}

\def\X{\mathbf X}

\def\B{\mathbf B}
\def\A{\mathbf A}

\def\D{\mathbf D}

\def\M{\mathbf M}

\def\Z{\mathbf Z}

\def\I{\mathbf I}
\def\Ic{\mathcal I}

\def\A{\mathbf A}

\def\Xt{\widetilde{\mathbf{X}}}

\def\E{\mathbb E}
\def\R{\mathbb R} 
\def\Pr{\mathrm{Pr}} 
\def\tr{\mathrm{tr}}

\let\origtop\top
\renewcommand\top{{\scriptscriptstyle{\origtop}}} 

\definecolor{silver}{cmyk}{0,0,0,0.3}
\definecolor{yellow}{cmyk}{0,0,0.9,0.0}
\definecolor{reddishyellow}{cmyk}{0,0.22,1.0,0.0}
\definecolor{black}{cmyk}{0,0,0.0,1.0}
\definecolor{darkYellow}{cmyk}{0.2,0.4,1.0,0}
\definecolor{orange}{cmyk}{0.0,0.7,0.9,0}
\definecolor{darkSilver}{cmyk}{0,0,0,0.1}
\definecolor{grey}{cmyk}{0,0,0,0.5}
\definecolor{darkgreen}{cmyk}{0.6,0,0.8,0}
\newcommand{\Red}[1]{{\color{red}  {#1}}}

\newcommand{\Green}[1]{{\color{darkgreen}  {#1}}}

\newcommand{\white}[1]{{\textcolor{white}{#1}}}

\ifx\proof\undefined
\newenvironment{proof}{\par\noindent{\bf Proof\ }}{\hfill\BlackBox\\[2mm]}
\fi

\ifx\theorem\undefined
\newtheorem{theorem}{Theorem}
\fi

\ifx\example\undefined
\newtheorem{example}{Example}
\fi

\ifx\condition\undefined
\newtheorem{condition}{Condition}
\fi
\ifx\property\undefined
\newtheorem{property}{Property}
\fi

\ifx\lemma\undefined
\newtheorem{lemma}[theorem]{Lemma}
\fi

\ifx\proposition\undefined
\newtheorem{proposition}[theorem]{Proposition}
\fi

\ifx\remark\undefined
\newtheorem{remark}[theorem]{Remark}
\fi

\ifx\corollary\undefined
\newtheorem{corollary}[theorem]{Corollary}
\fi

\ifx\definition\undefined
\newtheorem{definition}{Definition}
\fi

\ifx\conjecture\undefined
\newtheorem{conjecture}[theorem]{Conjecture}
\fi

\ifx\axiom\undefined
\newtheorem{axiom}[theorem]{Axiom}
\fi

\ifx\claim\undefined
\newtheorem{claim}[theorem]{Claim}
\fi

\ifx\assumption\undefined
\newtheorem{assumption}[theorem]{Assumption}
\fi

\ifx\condition\undefined
\newtheorem{condition}[theorem]{Condition}
\fi

%% file: bayesian_experiments.tex
\section{Experiments}\label{s:experiments}
We confirm our theoretical
results with experiments on real world data from \texttt{libsvm} datasets
\cite{libsvm} (more details in Appendix~\ref{a:experiments}).
For all our experiments, the prior precision matrix is set to $\A = n^{-1} \I$
and we consider sample sizes $k \in [d, 5d]$. Each experiment is
averaged over 25 trials and bootstrap 95\% confidence intervals are shown.
The quality of our method, as measured by the A-optimality
criterion $f_{\A}(\X_S^\top \X_S) = \tr \left((\X_S^\top \X_S + \A)^{-1}\right)$
is compared against the following references and recently proposed methods
for A-optimal design:
\begin{description}
    \item[Greedy bottom-up] adds an index $i \in [n]$ to the sample $S$
        maximizing the increase in A-optimality criterion
        \cite{greedy-supermodular,chamon2017approximate}.

    \item[Our method (with SDP)] uses the efficient algorithms
        developed in proving Theorem~\ref{t:q2} to sample
        $\DPPreg{p}(\X,\A)$ constrained to subset size $k$
        with $p = w^*$, see \eqref{eq:sdp},
        obtained using a         recently developed first order convex cone solver called Splitting
        Conical Solver (SCS) \cite{o2016conic}.
        We chose SCS because it can handle the SDP constraints in
        \eqref{eq:sdp} and has provable termination guarantees, while
        also finding solutions faster \cite{o2016conic} than alternative
        off-the-shelf optimization software libraries such as SDPT3 and Sedumi.

    \item[Our method (without SDP)] samples $\DPPreg{p}(\X,\A)$ with uniform
      probabilities $p \equiv \frac{k}{n}$.

    \item[Uniform] samples every size $k$ subset $S \subseteq [n]$
        with equal probability.

    \item[Predictive length] sampling \cite{zhu2015optimal} samples
        each row $\x_i$ of $\X$ with probability $\propto\|\x_i\|$.
\end{description}

\begin{figure}[htbp]
    \centering
    \hspace{-0.35cm}
    \includegraphics[width=\textwidth]{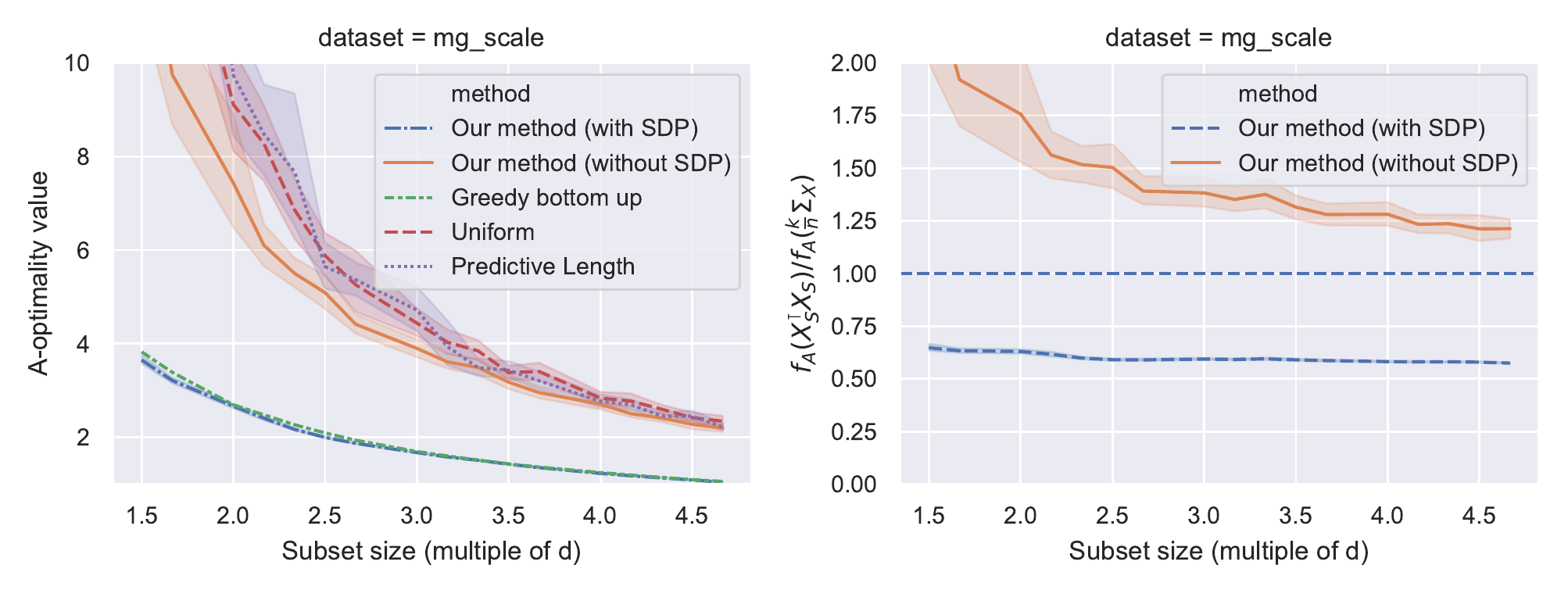}
    \caption{(left) A-optimality value obtained by the various methods on
        the \texttt{mg\_scale} dataset \cite{libsvm} with
        prior precision $\A = 10^{-5}\, \I$,\quad (right)
A-optimality value for our method (with and without SDP) divided by
$f_{\A}(\frac kn\Sigmab_{\X})$, the baseline estimate suggested by Theorem \ref{t:q1}.}
    \label{f:experiments}
\end{figure}

Figure~\ref{f:experiments} (left) reveals that our method (without SDP) is superior
to both uniform and predictive length sampling, producing designs which
achieve lower $A$-optimality criteria values for all sample sizes.
As Theorem~\ref{t:algorithm} shows that our method (without SDP) only differs
from uniform sampling by an additional DPP sample with controlled
expected size (see Lemma~\ref{l:size}), we may conclude
that adding even a small DPP sample can improve a uniformly sampled design.

Consistent with prior observations
\cite{chamon2017approximate,tractable-experimental-design}, the greedy bottom up
method achieves surprisingly good performance. However, if our method is used
in conjunction with an SDP solution, then we are able to match and
even slightly exceed the performance of the greedy bottom up
method. Furthermore, the overall run-time costs (see Appendix \ref{a:experiments})
between the two are comparable. As the majority of the runtime of our
method (with SDP) is occupied by solving the SDP, an interesting future direction
is to investigate alternative solvers such as interior point methods as well
as terminating the solvers early once an approximate solution is reached.

Figure~\ref{f:experiments} (right) displays the ratio
$f_{\A}(\X_S^\top \X_S)\, /\, f_{\A}(\frac{k}{n}\Sigmab_\X)$ for
subsets returned by our method (with and without SDP). Note that the
line for our method with SDP on
Figure~\ref{f:experiments} (right) shows that the ratio never goes
below 0.5, and we saw similar behavior across all examined datasets
(see Appendix~\ref{a:experiments}). This evidence suggests that for
    many real datasets $\opt$ is within a small constant factor of 
    $f_{\A}(\frac{k}{n}\Sigmab_\X)$, matching the upper bound of Theorem~\ref{t:q1}.

%% file: bayesian_experiments-appendix.tex
This section presents additional details and experimental results omitted from
the main body of the paper. In addition to the \texttt{mg\_scale} dataset presented in
Section~\ref{s:experiments}, we also benchmarked on three other data sets
described in Table~\ref{tab:libsvm-datasets}.
\begin{table}[ht]
    \centering
    \caption{\cite{libsvm} datasets used in experiments}
    \label{tab:libsvm-datasets}
    \begin{tabular}[t]{lcccc}
        \toprule
        & \texttt{mg\_scale} & \texttt{bodyfat\_scale} & \texttt{mpg\_scale} & \texttt{housing\_scale} \\
        \midrule
        $n$ & 1385 & 252 & 392 & 506 \\
        $d$ & 6 & 14 & 7 & 13\\
        \bottomrule
    \end{tabular}
\end{table}

The A-optimality values obtained are illustrated in Figure~\ref{f:obj-grid}.
The general trend observed in Section~\ref{s:experiments} of our method
(without SDP) outperforming independent sampling methods (uniform and
predictive length) and our method (with SDP) matching the performance of the
greedy bottom up method continues to hold across the additional datasets considered.

\begin{figure}[htpb]
    \centering
    \includegraphics[width=\textwidth]{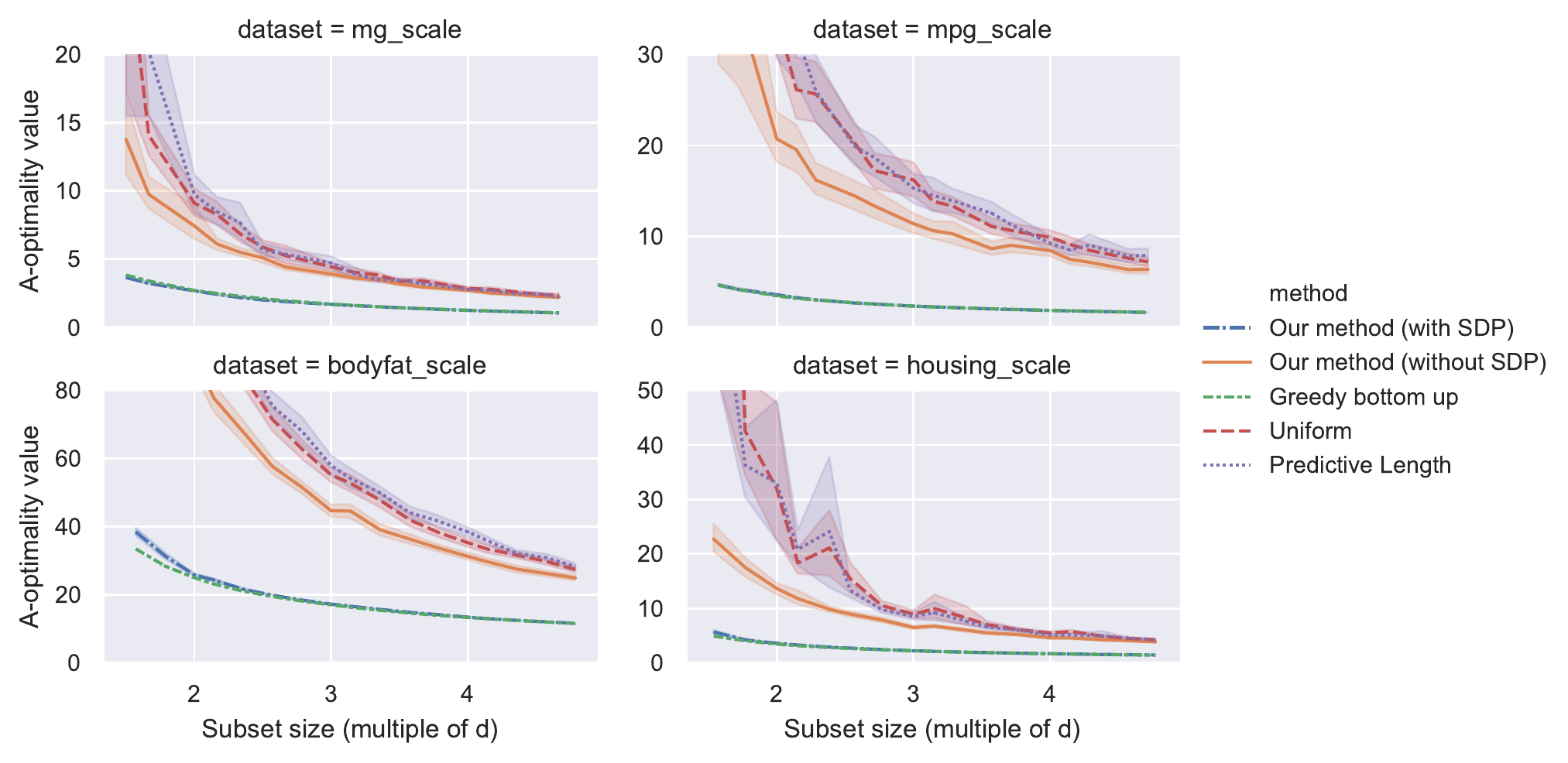}
    \caption{A-optimality values achieved by the methods compared. In all cases
        considered, we found our method (without SDP) to be superior to independent
        sampling methods like uniform and predictive length sampling. After paying the price
        to solve an SDP, our method (with SDP) is able to consistently match the performance
        of a greedy method which has been noted
        \cite{chamon2017approximate} to work well empirically.} 
    \label{f:obj-grid}
\end{figure}

The relative ranking and overall order of magnitude differences
between runtimes (Figure~\ref{f:runtimes}) are also similar across the various
datasets. An exception to the rule is on $\texttt{mg\_scale}$, where we see
that our method (without SDP) costs more than the greedy method
(whereas everywhere else it costs~less).

\begin{figure}[htpb]
    \centering
    \includegraphics[width=\textwidth]{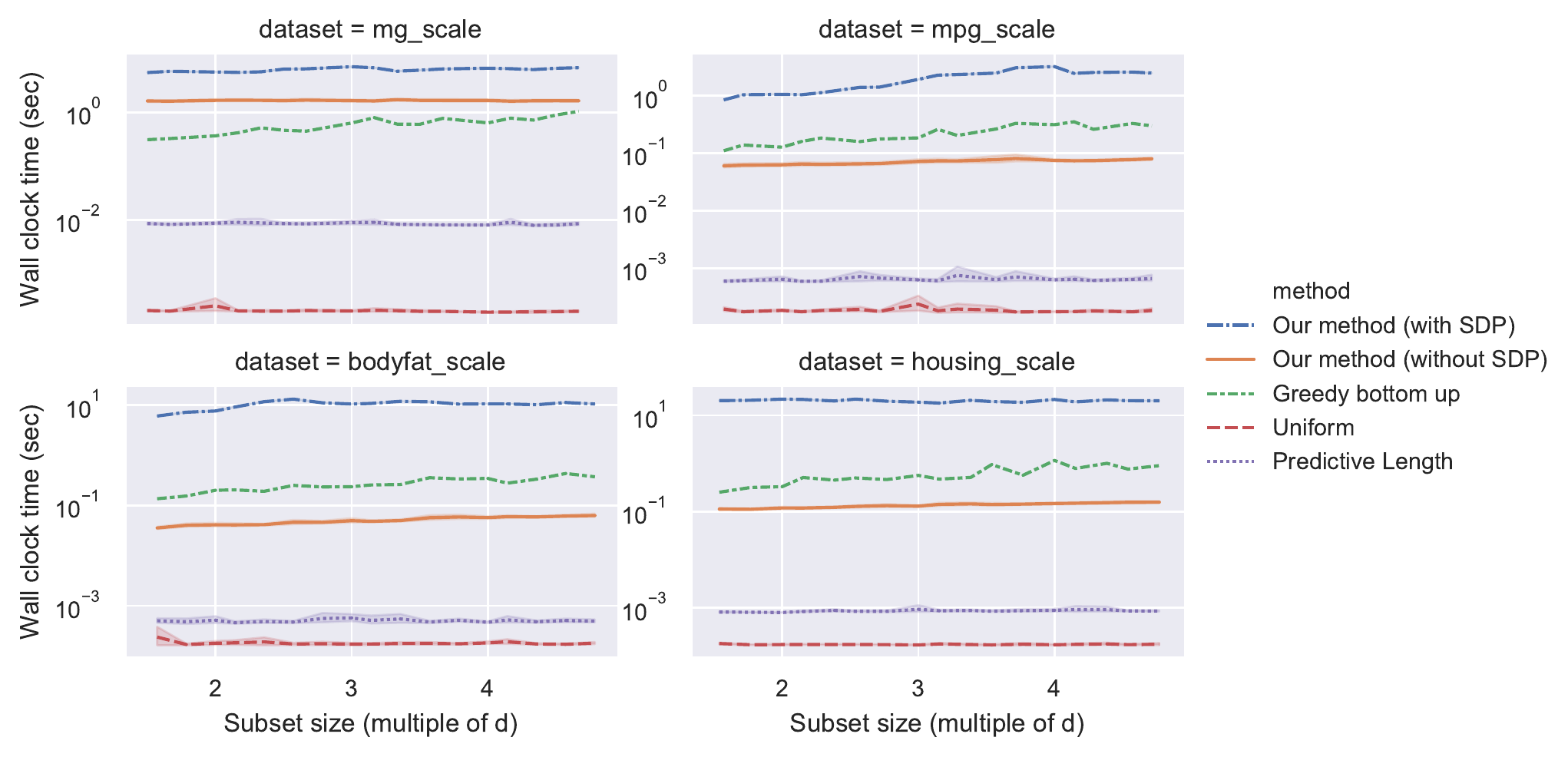}
    \caption{Runtimes of the methods compared. Our method (without SDP) is
        within an order of magnitude of greedy bottom up and faster in 3 out of
        4 cases. The gap between our method with and without SDP is
        attributable to the SDP solver, making investigation of more efficient
        solvers and approximate solutions an interesting direction for future
        work.
    }
    \label{f:runtimes}
\end{figure}

The claim that $f_{\A}(\frac{k}{n} \Sigmab_\X)$ is an appropriate
quantity to summarize the contribution of problem-dependent factors
on the performance of Bayesian A-optimal designs is further evidenced in Figure~\ref{f:ratios}.
Here, we see that after normalizing the A-optimality values by this
quantity, the remaining quantities are all on the same scale and close to $1$.

\begin{figure}[htpb]
    \centering
    \includegraphics[width=\textwidth]{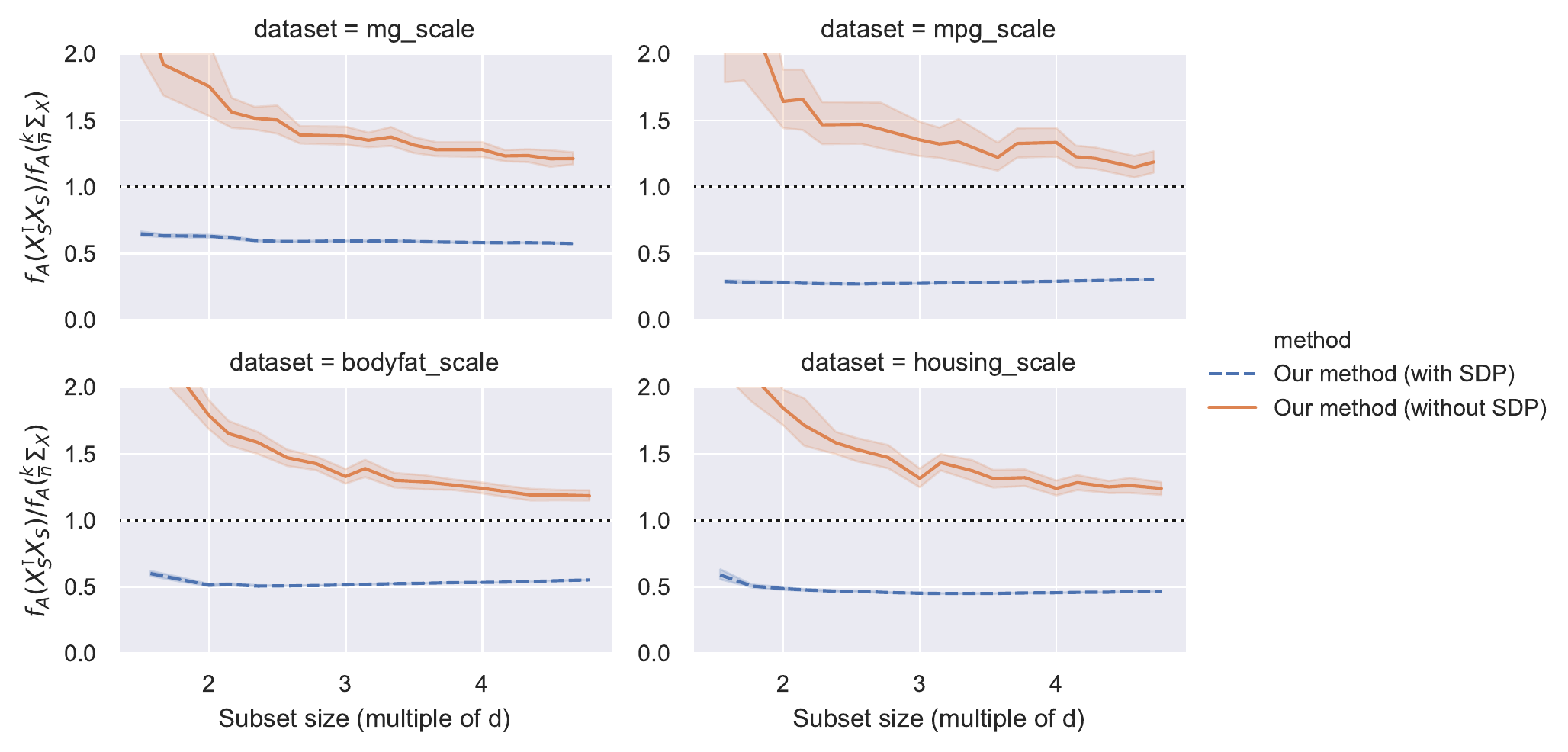}
    \caption{The ratio controlled by Lemma~\ref{l:guarantees}. This ratio converges
        to $1$ as $k \to n$ and is close to $1$ across all
        real world datasets,
        suggesting that $f_{\A}(\frac{k}{n} \Sigmab_\X)$
        is an appropriate problem-dependent scale for Bayesian A-optimal
        experimental design.
    }
    \label{f:ratios}
\end{figure}

%% file: bayesian.bbl
\newcommand{\etalchar}[1]{$^{#1}$}
\begin{thebibliography}{AZLSW17}

\bibitem[AB13]{avron-boutsidis13}
Haim Avron and Christos Boutsidis.
\newblock Faster subset selection for matrices and applications.
\newblock {\em SIAM Journal on Matrix Analysis and Applications},
  34(4):1464--1499, 2013.

\bibitem[AM15]{ridge-leverage-scores}
Ahmed~El Alaoui and Michael~W. Mahoney.
\newblock Fast randomized kernel ridge regression with statistical guarantees.
\newblock In {\em Proceedings of the 28th International Conference on Neural
  Information Processing Systems}, pages 775--783, Montreal, Canada, December
  2015.

\bibitem[AZLSW17]{near-optimal-design}
Zeyuan Allen-Zhu, Yuanzhi Li, Aarti Singh, and Yining Wang.
\newblock Near-optimal design of experiments via regret minimization.
\newblock In {\em Proceedings of the 34th International Conference on Machine
  Learning}, volume~70 of {\em Proceedings of Machine Learning Research}, pages
  126--135, Sydney, Australia, August 2017.

\bibitem[BBKT17]{greedy-supermodular}
Andrew~An Bian, Joachim~M. Buhmann, Andreas Krause, and Sebastian Tschiatschek.
\newblock Guarantees for greedy maximization of non-submodular functions with
  applications.
\newblock In Doina Precup and Yee~Whye Teh, editors, {\em Proceedings of the
  34th International Conference on Machine Learning}, volume~70 of {\em
  Proceedings of Machine Learning Research}, pages 498--507, International
  Convention Centre, Sydney, Australia, 06--11 Aug 2017. PMLR.

\bibitem[Ber11]{matrix-mathematics}
Dennis~S. Bernstein.
\newblock {\em Matrix Mathematics: Theory, Facts, and Formulas}.
\newblock Princeton University Press, second edition, 2011.

\bibitem[BGS10]{submodularity-optimal-design}
Mustapha Bouhtou, Stéphane Gaubert, and Guillaume Sagnol.
\newblock Submodularity and randomized rounding techniques for optimal
  experimental design.
\newblock {\em Electronic Notes in Discrete Mathematics}, 36:679--686, 08 2010.

\bibitem[BV04]{boyd2004convex}
Stephen Boyd and Lieven Vandenberghe.
\newblock {\em Convex optimization}.
\newblock Cambridge university press, 2004.

\bibitem[CL11]{libsvm}
Chih-Chung Chang and Chih-Jen Lin.
\newblock {LIBSVM}: A library for support vector machines.
\newblock {\em ACM Transactions on Intelligent Systems and Technology},
  2:27:1--27:27, 2011.

\bibitem[CN80]{cook1980comparison}
R~Dennis Cook and Christopher~J Nachtrheim.
\newblock A comparison of algorithms for constructing exact d-optimal designs.
\newblock {\em Technometrics}, 22(3):315--324, 1980.

\bibitem[CR17a]{chamon2017approximate}
Luiz Chamon and Alejandro Ribeiro.
\newblock Approximate supermodularity bounds for experimental design.
\newblock In {\em Advances in Neural Information Processing Systems}, pages
  5403--5412, 2017.

\bibitem[CR17b]{greedy-graph-sampling}
Luiz F.~O. Chamon and Alejandro Ribeiro.
\newblock Greedy sampling of graph signals.
\newblock {\em CoRR}, abs/1704.01223, 2017.

\bibitem[CV95]{bayesian-design-review}
Kathryn Chaloner and Isabella Verdinelli.
\newblock Bayesian experimental design: A review.
\newblock {\em Statist. Sci.}, 10(3):273--304, 08 1995.

\bibitem[DCMW19]{minimax-experimental-design}
Micha{\l} {Derezi{\'n}ski}, Kenneth~L. {Clarkson}, Michael~W. {Mahoney}, and
  Manfred~K. {Warmuth}.
\newblock {Minimax experimental design: Bridging the gap between statistical
  and worst-case approaches to least squares regression}.
\newblock In {\em Proceedings of the 32nd Conference on Learning Theory}, 2019.

\bibitem[Der19]{dpp-intermediate}
Micha{\l} Derezi\'{n}ski.
\newblock Fast determinantal point processes via distortion-free intermediate
  sampling.
\newblock In {\em Proceedings of the 32nd Conference on Learning Theory}, 2019.

\bibitem[DM16]{DM16_CACM}
Petros Drineas and Michael~W. Mahoney.
\newblock {RandNLA}: Randomized numerical linear algebra.
\newblock {\em Communications of the ACM}, 59:80--90, 2016.

\bibitem[DM17]{RandNLA_PCMIchapter_TR}
Petros Drineas and Michael~W. Mahoney.
\newblock Lectures on randomized numerical linear algebra.
\newblock Technical report, 2017.
\newblock Preprint: arXiv:1712.08880; To appear in: \emph{Lectures of the 2016
  PCMI Summer School on Mathematics of Data}.

\bibitem[DM19]{determinantal-averaging}
Micha{\l} {Derezi{\'n}ski} and Michael~W. {Mahoney}.
\newblock {Distributed estimation of the inverse Hessian by determinantal
  averaging}.
\newblock {\em arXiv e-prints}, page arXiv:1905.11546, May 2019.

\bibitem[DW17]{unbiased-estimates}
Micha{\l} Derezi\'{n}ski and Manfred~K. Warmuth.
\newblock Unbiased estimates for linear regression via volume sampling.
\newblock In {\em Advances in Neural Information Processing Systems 30}, pages
  3087--3096, Long Beach, CA, USA, December 2017.

\bibitem[DW18a]{unbiased-estimates-journal}
Micha{\l} Derezi{\'n}ski and Manfred~K. Warmuth.
\newblock Reverse iterative volume sampling for linear regression.
\newblock {\em Journal of Machine Learning Research}, 19(23):1--39, 2018.

\bibitem[DW18b]{regularized-volume-sampling}
Micha{\l} Derezi\'{n}ski and Manfred~K. Warmuth.
\newblock Subsampling for ridge regression via regularized volume sampling.
\newblock In Amos Storkey and Fernando Perez-Cruz, editors, {\em Proceedings of
  the Twenty-First International Conference on Artificial Intelligence and
  Statistics}, pages 716--725, Playa Blanca, Lanzarote, Canary Islands, April
  2018.

\bibitem[DWH18]{leveraged-volume-sampling}
Micha{\l} Derezi\'{n}ski, Manfred~K. Warmuth, and Daniel Hsu.
\newblock Leveraged volume sampling for linear regression.
\newblock In S.~Bengio, H.~Wallach, H.~Larochelle, K.~Grauman, N.~Cesa-Bianchi,
  and R.~Garnett, editors, {\em Advances in Neural Information Processing
  Systems 31}, pages 2510--2519. Curran Associates, Inc., 2018.

\bibitem[DWH19]{correcting-bias}
Micha{\l} {Derezi{\'n}ski}, Manfred~K. {Warmuth}, and Daniel {Hsu}.
\newblock {Correcting the bias in least squares regression with volume-rescaled
  sampling}.
\newblock In {\em Proceedings of the 22nd International Conference on
  Artificial Intelligence and Statistics}, 2019.

\bibitem[HKP{\etalchar{+}}06]{dpp-independence}
J.~Ben Hough, Manjunath Krishnapur, Yuval Peres, B{\'a}lint Vir{\'a}g, et~al.
\newblock Determinantal processes and independence.
\newblock {\em Probability surveys}, 3:206--229, 2006.

\bibitem[KT12]{dpp-ml}
Alex Kulesza and Ben Taskar.
\newblock {\em Determinantal Point Processes for Machine Learning}.
\newblock Now Publishers Inc., Hanover, MA, USA, 2012.

\bibitem[Mah11]{Mah-mat-rev_JRNL}
Michael~W. Mahoney.
\newblock Randomized algorithms for matrices and data.
\newblock {\em Foundations and Trends in Machine Learning}, 3(2):123--224,
  2011.
\newblock Also available at: arXiv:1104.5557.

\bibitem[NSTT19]{proportional-volume-sampling}
Aleksandar Nikolov, Mohit Singh, and Uthaipon Tao~Tantipongpipat.
\newblock Proportional volume sampling and approximation algorithms for a
  -optimal design.
\newblock In {\em Proceedings of the Thirtieth Annual ACM-SIAM Symposium on
  Discrete Algorithms}, pages 1369--1386, January 2019.

\bibitem[OCPB16]{o2016conic}
Brendan O’Donoghue, Eric Chu, Neal Parikh, and Stephen Boyd.
\newblock Conic optimization via operator splitting and homogeneous self-dual
  embedding.
\newblock {\em Journal of Optimization Theory and Applications},
  169(3):1042--1068, 2016.

\bibitem[Puk06]{optimal-design-pukelsheim}
Friedrich Pukelsheim.
\newblock {\em Optimal Design of Experiments (Classics in Applied Mathematics)
  (Classics in Applied Mathematics, 50)}.
\newblock Society for Industrial and Applied Mathematics, Philadelphia, PA,
  USA, 2006.

\bibitem[WYS17]{tractable-experimental-design}
Yining Wang, Adams~W. Yu, and Aarti Singh.
\newblock On computationally tractable selection of experiments in
  measurement-constrained regression models.
\newblock {\em J. Mach. Learn. Res.}, 18(1):5238--5278, January 2017.

\bibitem[ZMMY15]{zhu2015optimal}
Rong Zhu, Ping Ma, Michael~W Mahoney, and Bin Yu.
\newblock Optimal subsampling approaches for large sample linear regression.
\newblock {\em arXiv preprint arXiv:1509.05111}, 2015.

\end{thebibliography}
